%% file: main.tex
\PassOptionsToPackage{square,comma,numbers,sort&compress}{natbib}
\documentclass[twoside]{article}

\usepackage[utf8]{inputenc} 
\usepackage{booktabs}       
\usepackage{amsfonts}       
\usepackage{nicefrac}       
\usepackage{microtype}      
\newsavebox{\measurebox}

\usepackage{bm}
\usepackage{soul}
\usepackage{amsmath}
\usepackage{amsthm}
\usepackage{bbding} 
\usepackage[ruled,vlined,shortend, linesnumbered]{algorithm2e} 
\usepackage{enumitem}
\usepackage{subcaption}
\usepackage{amsfonts}
\usepackage{algpseudocode}

\usepackage{times}  
\usepackage{helvet} 
\usepackage{courier}  
\usepackage[hyphens]{url}  
\usepackage{graphicx} 
\usepackage{multirow}
\usepackage{wrapfig}
\usepackage{definitions}
\usepackage[hidelinks]{hyperref}       
\usepackage[table]{xcolor}

\SetCommentSty{mycommfont}
\SetKwInput{KwInput}{Input}                
\SetKwInput{KwOutput}{Output}              

\newtheorem*{theorem*}{Theorem}

\newtheorem*{remark*}{Remark}
\newcommand{\deepmethod}{\textsc{DeepCAC}}

\newcommand{\blue}[1]{\textcolor{blue}{#1}}

\usepackage[accepted]{aistats2023}
\usepackage[round]{natbib}

\begin{document}

%

%

\setcounter{secnumdepth}{2} 

\twocolumn[

\aistatstitle{Clustering Aware Classification for Risk Prediction \& Subtyping in Clinical Data}

\aistatsauthor{Shivin Srivastava\textsuperscript{1}\And
Siddharth Bhatia\textsuperscript{1} \And
Lingxiao Huang\textsuperscript{2} \And
Jun Heng Lim\textsuperscript{3} \And
Kenji Kawaguchi\textsuperscript{1} \And
Vaibhav Rajan\textsuperscript{1}}

\aistatsaddress{\textsuperscript{1}School of Computing, National University of Singapore, Singapore \\
\textsuperscript{2} Huawei TCS Labs, Shanghai, China \\
\textsuperscript{3} Nanyang Institute of Technology, Singapore} ]

\begin{abstract}
\looseness=-1
In data containing heterogeneous subpopulations, classification performance benefits from incorporating the knowledge of cluster structure in the classifier. Previous methods for such combined clustering and classification either 1) are classifier-specific and not generic, or 2) independently perform clustering and classifier training, which may not form clusters that can potentially benefit classifier performance. The question of how to perform clustering to improve the performance of classifiers trained on the clusters has received scant attention in previous literature, despite its importance in several real-world applications. In this paper, first, we theoretically analyze the generalization performance of classifiers trained on clustered data and find conditions under which clustering can potentially aid classification. This motivates the design of a simple $k$-means-based classification algorithm called \textbf{C}lustering \textbf{A}ware \textbf{C}lassification (CAC) and its neural variant {\bf \deepmethod}. \deepmethod\ effectively leverages deep representation learning to learn latent embeddings and finds clusters in a manner that make the clustered data suitable for training classifiers for each underlying subpopulation. Our experiments on synthetic and real benchmark datasets demonstrate the efficacy of \deepmethod\ over previous methods for combined clustering and classification.
\end{abstract}

\section{Introduction}
\label{sec:intro}
\looseness=-1
Many real datasets have complex underlying structures such as clusters and intrinsic manifolds.
For classification tasks on such data, linear classifiers fail to learn well because of the inherent non-linear data distribution.
Non-linear classifiers such as (deep) neural networks perform well in such cases, but often require large training data sets to achieve good generalization. When clusters are found or suspected in the data, it is well known that classification performance benefits from incorporating knowledge of such cluster structure in the model, for both linear and non-linear classifiers \cite{qian2012simultaneous,gu2013clustered,chakraborty2017ec3,li2020predicting}.

\looseness=-1
As an example, consider the problem of predicting the risk of disease using patients' clinical data.
Patient populations, even for a single disease, show significant clinical heterogeneity.
As a result, subpopulations (called subtypes) having relatively homogeneous clinical characteristics can often be found in the data.
Models that do not consider underlying subtypes may be consistently underestimating or overestimating the risks in specific subtypes
\cite{snyderman2012personalized}, and
risk models that account for subtypes are more effective  \cite{alaa2016personalized,suresh2018learning,li2020predicting,ibrahim2020classifying}.
Similar examples illustrate the benefit of clustering to build classifiers in, e.g., image processing \cite{kim2008mcboost,peikari2018cluster} and natural language processing \cite{alsmadi2015clustering}.

Previous works on using clustering for classification broadly fall into two categories.
The first group of techniques modify specific classification techniques to account for the underlying clusters, e.g., Clustered SVMs \cite{gu2013clustered}.
These techniques are closely tied to the modified classifier and inherit both its modeling strengths and weaknesses.
The second and more common approach is to first cluster the data and then independently train classifiers on each cluster.
Although simple and intuitive, such a `cluster-then-predict' approach
is not designed to form clusters
optimized to improve
the performance of classifiers trained on those clusters. 
To the best of our knowledge, the question of performing clustering to improve the performance of classifiers trained on the clusters has not been addressed by previous literature.

\looseness=-1
Thus, in this paper, we first investigate the fundamental question of when and how clustering can help in obtaining accurate classifiers.
Our analysis yields novel insights into the benefits of using simple
classifiers trained on clusters compared to simple or complex 
(in terms of Rademacher complexity \cite{mohri2012foundations}) classifiers trained without clustering. It provides clues on when such a clustering-based approach may or may not improve the subsequent classification.
{
This leads to the design of our first
clustering-based algorithm for classification, called Clustering Aware Classification (CAC). CAC
finds 
clusters that are intended to be used as training datasets by the (base) classifiers of choice e.g., logistic regression, SVM, or neural networks. 
The key idea of our approach is to induce class separability (polarization) within each cluster, for better downstream classification, in addition to partitioning the data.
This is effected by designing a cost function that can be used within the optimization framework of $k$-means.}

Theoretical and empirical analysis of CAC demonstrates the benefits of inducing class separability in clusters for downstream classification.
However, linear separability in CAC and clustering in the data space do not yield good quality clusters.
To address this limitation we turn to recent deep clustering methods that have shown the efficacy of clustering on low-dimensional embeddings obtained from deep representation learning, through autoencoders.
Adopting this approach for simultaneous clustering and classification is not straightforward as we face two hurdles.
First, the class-separability loss in CAC cannot be optimized directly during network training. 
Second, we often obtain degenerate clustering solutions with empty clusters.
We solve these two problems in our second algorithm \deepmethod, a neural variant of CAC, that uses an additive margin softmax loss to induce class separability in clusters and a normalization term to avoid degenerate solutions. In summary, our contributions are:

\vspace{-5mm}

\begin{enumerate}[noitemsep,topsep=0pt,labelindent=0em,leftmargin=*]
    \item {Our theoretical analysis provides insights into the fundamental question of how clustering can aid in potentially improving the performance of classifiers trained on the clusters}.
    
    \item We propose the concept of class separation aware clustering and provide a theoretical result showing how class separation bounds the log-loss of clusters.
    We design a simple, proof-of-concept algorithm, Clustering Aware Classification (CAC) based on $k$-means algorithm.
    We conduct 
    extensive empirical 
    studies 
    on synthetic data and real data to evaluate its efficacy and validate our theoretical results.

    \item We propose \deepmethod, a novel model for simultaneous clustering and classification that realizes cluster-specific class separation through additive margin softmax loss. We introduce novel techniques 
    to stabilize training and prevent degenerate clusters.

    \item Our experimental results on benchmark clinical datasets
    show that \deepmethod\ yields 
    improved performance over the state-of-art approaches for combined clustering and classification on large, highly imbalanced datasets while being comparable on mid-sized datasets.
    We also demonstrate the practical utility of our approach through a case study on the prediction of length of stay in Intensive Care Units in hospitals.
\end{enumerate}

\textbf{REPRODUCIBILITY}: Our code and datasets are publicly available at \href{https://www.dropbox.com/s/vh9cfq0qs7k4vou/CAC\_code.zip}{\blue{this anonymous link}}.

\section{Related Work}
\label{sec:related_work}
\looseness=-1
The idea of training multiple local classifiers on portions of the dataset has been extensively studied for specific classifiers \cite{hess2020softmax,elouedi2014hybrid,vincent2014sparse,visentin2019predicting,mahfouz2020ensemble,beluch2018power,titov2010unsupervised,hao2018simultaneous}.
For instance,
\cite{qi2011llsvm} presents a novel locality-sensitive support vector machine (LSSVM) for the image retrieval problem. Locally Linear SVM \cite{torr2011locally} has a smooth decision boundary and bounded curvature. Its authors show how functions defining the classifier can be approximated using any local coding scheme.

\looseness=-1
There are other related works, which aim to find suitable representations of raw data such that the performance in downstream prediction tasks is better when they use the learned embeddings. Large Margin Nearest Neighbour (LMNN) \cite{weinberger2009distance} introduces the much-celebrated triplet loss. In this work, the authors find a data-specific Mahalanobis distance embedding for the data, subject to the constraint that $k$ similarly labeled points are mapped in the neighborhood of every point. The transformed data is more amenable for a $k$ Nearest Neighbour classifier to work with. \cite{song2017parameter} extends LMNN to make it parameter-free. 
Some Bayesian approaches have been developed for combined classification and clustering, e.g., \cite{qian2012simultaneous,cai2009simultaneous} but they are computationally expensive.

Recent neural models for combined clustering and classification, that are closest to our work, are DMNN, GRASP and, AC-TPC.
In Deep Mixture Neural Network (DMNN) \cite{li2020predicting} neural representations, from an encoder, are clustered using softmax gating and a mixture of experts, comprising neural networks, is used to predict risk.
However, as seen in our experiments, clustering through softmax gating does not yield well-characterized subtypes and does not improve risk prediction.
GRASP \cite{zhang2021grasp} aims to boost healthcare models by incorporating auxiliary information from similar patients.
In GRASP, clustering and prediction processes influence and enhance each other over iterations and thus it can be viewed as an algorithm for simultaneous classification and clustering. 
In \cite{lee2020temporal} an Actor-Critic Approach for Temporal Predictive Clustering (AC-TPC) is designed wherein latent representations from patient data
are clustered ensuring intra-cluster homogeneity in outcomes. A predictor neural network is conditioned on cluster centroids, with actor-critic training.

\section{Background}
\label{sec:bgd}
Given a set of data samples $X = (x_1, x_2, \cdots, x_N)$ with $N$ records $x_i \in \mathbb{R}^{d}$ together with binary labels $y_i\in \left\{0,1\right\}$, the task of clustering is to group the $N$ points into $K$ clusters.
$K$-means is arguably the most widely adopted algorithm due to its simplicity, speed, properties that are well studied in literature \cite{lloyd1982least,coates2012kmeans}. It works by optimizing the following cost function:

\begin{equation}
\begin{aligned}
\min _{\boldsymbol{M} \in \mathbb{R}^{M \times K},\left\{\boldsymbol{s}_{i} \in \mathbb{R}^{K}\right\}} & \sum_{i=1}^{N}\left\|\boldsymbol{x}_{i}-\boldsymbol{M} \boldsymbol{s}_{i}\right\|_{2}^{2}\\
&\text {s.t.\ } s_{j, i} \in\{0,1\}, \mathbf{1}^{T} \boldsymbol{s}_{i}=1 \quad \forall i, j \nonumber
\end{aligned}
\end{equation}

where $\boldsymbol{s}_i$ is the assignment vector of data point $i$ which has
only one non-zero element, $s_{j,i}$ denotes the $j^{th}$ element of
$\boldsymbol{s}_i$, and the $k^{th}$ column of $M$ (the cluster assignment matrix), i.e., $m_k$, denotes the centroid of the $k^{th}$ cluster. $k$-means works well when the data points are evenly distributed about their centroids. Real world data, however, are generally high dimensional and thus not `$k$-means friendly'. 
The Deep Clustering Network (DCN) \cite{yang2017towards} algorithm solves this problem by designing a deep neural network (DNN) structure and the associated joint optimization criterion for  
joint dimensionality reduction (DR) and clustering instead of using DR as a mere preprocessing technique, also followed in other works, e.g. \cite{yang2016learning,yang2016joint}.
The loss function they optimize takes the form:
\[
\min_{\Ucal, \Wcal} \ell(X, g(f(X))) +  \beta \cdot \sum_{i=1}^{N} \|f(x_i) - Ms_i\|^2 + \gamma \cdot r(f(X))
\]

{where $X=[x_1,\cdots,x_N]$, $\beta>0$ is a parameter for balancing data fidelity and the latent cluster structure; $\ell$ is a function that measures the reconstruction error. Popular choices for $\ell$ include the KL divergence loss and least square loss; $r$ is a regularization imposed upon the latent embeddings to encourage other properties (like class separability in our case) weighted by $\gamma>0$}.

\section{Problem Statement}
\label{sec:problem}

Let $X = (x_1, x_2, \cdots, x_N)$ be a training dataset with $N$ records $x_i \in \mathbb{R}^{d}$ together with binary labels $y_i\in \left\{0,1\right\}$. 
%
Given an integer $k\geq 1$,
our goal is to:
\begin{enumerate}[noitemsep,topsep=0pt,labelindent=0em,leftmargin=*]
\item 
Divide $X$ into $k$ non-empty, distinct clusters $\Ccal = \{C_1, C_2, \cdots, C_k\}$, and
\item
Obtain a collection of $k$ classifiers $\Fcal = \{f_1, f_2, \cdots, f_k\}$ where each $f_j: C_j\rightarrow \left\{0,1\right\}^{|C_j|}$ is trained on all $x_i \in C_j$ with the objective
$
\argmin_{\Fcal, \Ccal} \sum_{j=1}^{k} \sum_{i \in C_j} \ell(f_j(x_i), y_i).
$
\end{enumerate}

There is no restriction on the choice of the classifier, which is determined apriori by the user.

\input{theory}

\section{\deepmethod}
\looseness=-1
\paragraph{Network Architecture:}
We propose to employ a Stacked Autoencoder (SAE) for dimensionality reduction. Figure \ref{fig:deepcac} illustrates the \deepmethod\ architecture. 
An SAE consists of an encoder (E) and decoder (D) parameterized by $(\Ucal, \Vcal)$ as $E(\Ucal): X \xrightarrow{} Z$ and $D(\Vcal): Z \xrightarrow{} X$. The encoder transforms the raw data into the embedded data space $Z$ (or the `bottleneck' layer). The `decoding' layers that try to reconstruct the data from the latent space are on the right-hand side. 
In the subsequent classification stage after the clustering stage,
there exist $k$ local networks $(\{LN_j(\Wcal_j)\}_{j=1}^{k})$ that are trained separately from the autoencoder.
Similar to CAC, the data points are clustered but now in the embedded data space $Z$.
Note that, unlike cluster-then-predict approaches, clustering is not independent of classification since class labels are used during clustering as described below.

\begin{figure}[!htb]
\begin{tabular}{cc}
\subfloat[SIL score for different algorithms]{\includegraphics[width = 0.23\textwidth]{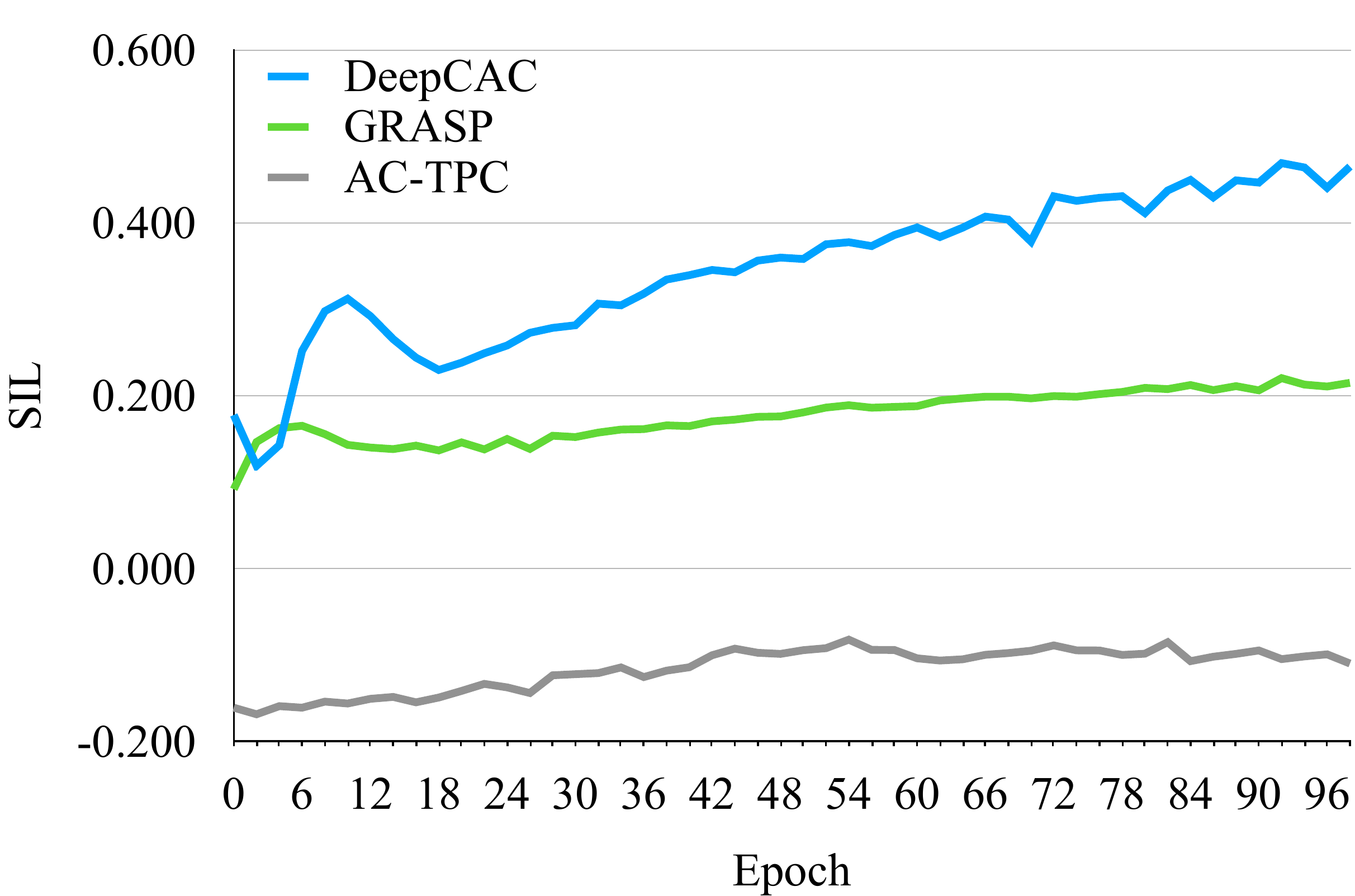}} &

\subfloat[AUC/AUPRC/SIL vs $\alpha$]{\includegraphics[width = 0.23\textwidth]{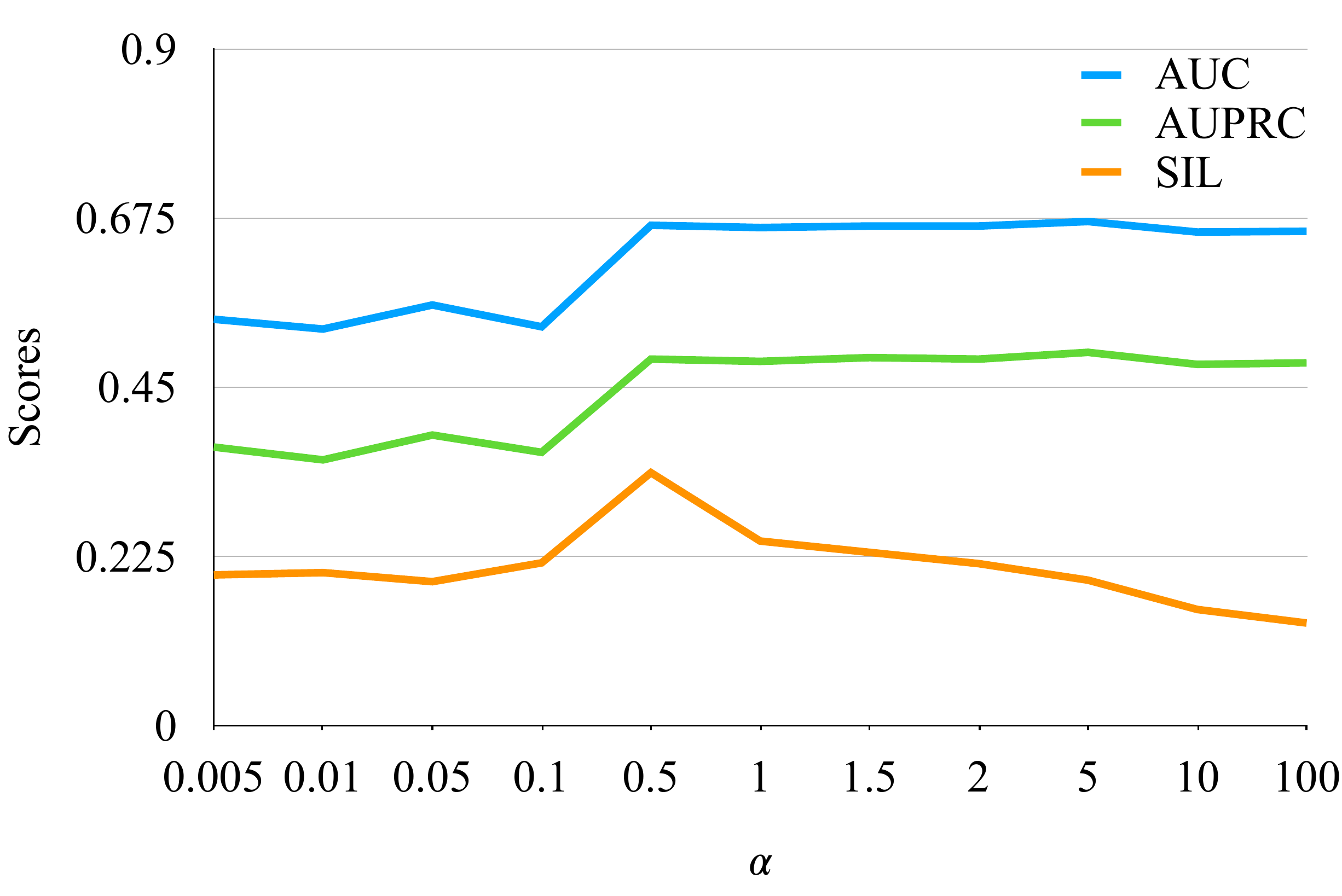}}

\end{tabular}
\caption{
(a) Silhouette scores vs. epochs. (b) Sensitivity Analysis with AUC/AUPRC and Silhouette scores (on \deepmethod) on CIC dataset. $k=3$ in all experiments.}
\label{fig:sensitivity_analysis}
\end{figure}

\paragraph{Loss Function:}
\looseness=-1
The \deepmethod\ loss function consists of 3 parts: The auto-encoder reconstruction error term ($\|x_i - D(E(x_i)) \|^2$) that regularizes the encoder as suggested by \cite{Guo2017idec}. The sum of squared errors (SSE) term $\left(\frac{\|E(x_i) - Ms_i \|^2}{|C_j|-1+\delta}\right)$ that encourages points to lie to close to their cluster centroid thus forming $k$-means friendly clusters, and a third term that we explain in the next paragraph. Unlike previous approaches \cite{caron2020unsupervised} that employ ad-hoc methods like cluster resampling in the case of an empty cluster, we normalize the cluster SSE (dividing it by the size of cluster) to stabilize training and prevent degenerate solutions.

The CAC loss function encourages cluster structure formation and intra-cluster class separation by directly optimizing for the distance between class cluster centroids. 
The class separation term in CAC encourages a gap between individual class points in respective clusters.
But we observed that \deepmethod\ was not able to directly optimize such a formulation of class separability through gradient descent. So we replace the CAC class separation term with the well-known Additive Margin Softmax Loss (AM Softmax) \cite{wang2018additive} as another way of inducing class separability within clusters.
The AM Softmax loss includes an additive margin to the standard softmax loss to push the classification boundary closer to the weight vector of each class. The final loss function of \deepmethod\ thus becomes:
\begin{align}
\label{eqn:deepcac_loss}
    L &= \sum_{j=1}^{k} \sum_{i\in C_j} \|x_i - D(E(x_i)) \|^2 + \frac{\beta*\|E(x_i) - Ms_i \|^2}{|C_j|-1+\delta} - \\  \nonumber
    & \frac{\alpha}{|C_j|} \sum_{i}^{} \log \frac{e^{s \cdot\left(W_{y_{i}}^{T} z_{i}-m\right)}}{e^{s \cdot\left(W_{y_{i}}^{T} z_{i}-m\right)}+\sum_{j=1, j \neq y_{i}}^{c} e^{s W_{j}^{T} z_{i}}}
    \\
& \text{s.t.}\ s_{j,i} \in \{0,1\}, \mathbf{1}^{T}s_i = 1 \ \forall i, j \nonumber
\end{align}

where $W^{T}_{j}$ is the $j$-th column of the last fully connected layer of the AMSoftmax loss. $M$ denotes the centroid matrix and $s_{i,j}$ is the assignment vector of data point $x_i$ which has only one non-zero element. $s_{j,i}$ denotes the $j^{th}$ element of $s_i$. $\mu_k$, the $k^{th}$ column of $M$, denotes the centroid of the $k^{th}$ cluster, $z_i$ is the embedding for $x_i$ and $|C_j|$ is the cardinality of cluster $C_j$.

\paragraph{How does AM Softmax loss induce class separability?}

\begin{theorem}[Proof in Appendix \ref{app:bounding_ams_loss}]
\label{thm:ams_cac}
Let $X = (x_1, x_2, \cdots, x_N)$ be a training dataset with $N$ records $x_i \in \mathbb{R}^{d}$ together with binary labels $y_i\in \left\{0,1\right\}$. Define the AM Softmax loss over the entire dataset as : $-\sum_{i}^{} \log \frac{e^{s \cdot\left(W_{y_{i}}^{T} z_{i}-m\right)}}{e^{s \cdot\left(W_{y_{i}}^{T} z_{i}-m\right)}+\sum_{j=1, j \neq y_{i}}^{c} e^{s W_{j}^{T} z_{i}}}$. Then,
\[
K_1 - \Gamma(K_3 \mu^{+} - K_4\mu^{-}) \leq L_{AMS}(X) \leq K_2 - \Gamma(K_3\mu^{+} - K_4\mu^{-})
\]
where $K_1 = N\left(\log2 + \frac{sm}{2}\right),\ K_2 = N(1 + sm),\ K_3 = sN^{+},\ K_4 = sN^{-},\ N^{+} = \sum_{x_i \in X} y_i$, $N^{-} = \sum_{x_i \in X} (1-y_i)$ and $\Gamma = W^T_{+} - W^T_{-}$.
\end{theorem}

The upper and lower bounds for the AM-softmax loss are strikingly similar to those of the logistic loss found in theorem \ref{thm:bounding_log_loss}. We note that as $L_{AMS}$ decreases, it pushes the lower bound to become smaller. The agent vectors ($W_{+/-}$) will lie nearby the class centroids ($\mu^{+/-}$) \cite{wang2017normface} and can thus be approximated as $W_{+} \sim \mu^{+}$ and $W_{-} \sim \mu^{-}$. The lower bound for $L_{AMS}$ can then be simplified as $K_1-\Gamma(K_3\mu^{+}-K_4\mu^{-}) = K_1 - sN(1-\mu^{+}\cdot\mu^{-})$ (since $|\mu^{+}| = |\mu^{-}| = 1$).
\looseness=-1
This implies that as $L_{AMS}$ pushes down the lower bound, $\mu^{+}\cdot\mu^{-}$ must decrease i.e. class centroids move further away (separate) from each other on the $L_2$ hypersphere. From theorem \ref{thm:bounding_log_loss}, we know that an increasing class separation decreases the training log-loss. Thus replacing CAC loss with the AM-Softmax loss achieves the same purpose.

\begin{figure*}[t]
    \centering
    \makebox[\textwidth][c]{\includegraphics[width=0.85\linewidth]{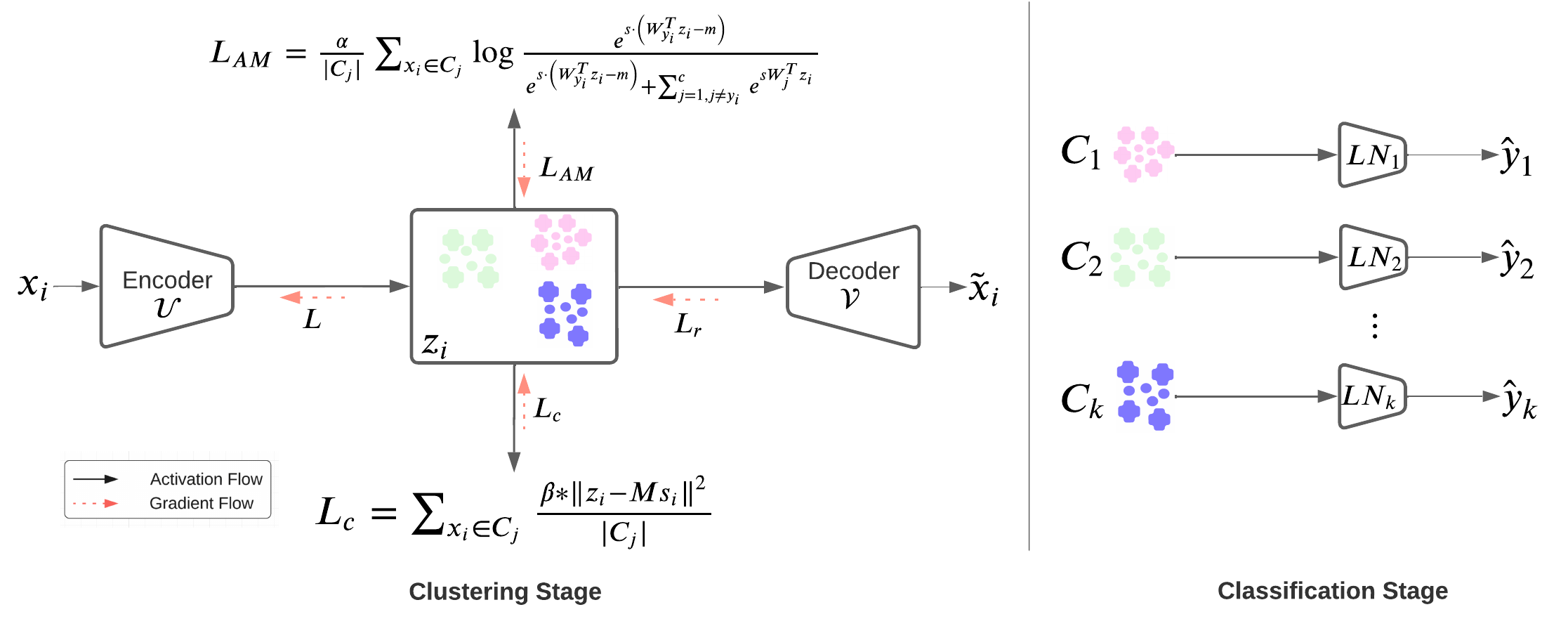}}%
    \caption{In the clustering stage an encoder is trained to find representations that are clustered with
    good class separability. In the classification stage, 
    cluster-specific classifiers are trained on the representations. $L_{AM}$ and $L_C$ are defined in eqn. \ref{eqn:deepcac_loss}.}
    \label{fig:deepcac}
\end{figure*}

\paragraph{Training:}
\looseness=-1
\deepmethod\ is trained in 3 stages: (i) Pre-training, (ii) Training the clustering network and (iii) Training the local networks.
During pre-training
the autoencoder parameters $(\Ucal, \Vcal)$ are initialized by training the autoencoder to minimize the reconstruction loss without activating the \deepmethod loss functions. Once the autoencoder is initialized, the data embeddings are clustered using the $k$-means algorithm and the cluster centroids are initialized.
The overall \deepmethod\ loss function is 
non-convex and difficult to optimize.
Stochastic gradient descent (SGD) cannot be directly applied to jointly optimize $\Ucal, \Vcal, M$, and $\{s_i\}$ because the block variable $\{s_i\}$ is constrained on a discrete set. We thus 
use an alternating scheme, similar to \cite{yang2017towards}, to
optimize the subproblems w.r.t. one of $M, \{s_i\}$ and $(\Ucal, \Vcal)$ while keeping the other two sets of variables fixed.
Further details regarding the updating of network parameters can be found in Appendix \ref{app:deepcac_training}.
Finally, the local networks are trained on the clusters found by the encoder in the representation space. Algorithm \ref{algo:deepcac} summarizes the entire procedure.

\begin{algorithm}[!hbt]
    \caption{\deepmethod\ \label{algo:deepcac}}
    \KwIn{Training Data: $X \in \mathbb{R}^{n\times d}$, class labels $y^{n \times 1} \in [\mathcal{B}]^{n}$, $\{C_j\}_1^{k}$, $\{f_j\}_1^{k}$}
    {\bf $\triangleright$ Initialization} \\
    Pretrain $E$ and $D$ to initialize $\Ucal, \Vcal$ \\
    Compute $\mu(C_j) \forall\ C_j\ s.t.\ j \in \{1, \cdots, k\}$.\\
    {\bf $\triangleright$ Algorithm} \\
    \While{not converged}{
    \For{every mini-batch $\Xcal_b$}{
        Update Network parameters by backpropagating $L$ (Eq. \ref{eqn:deepcac_loss})\\
        Update cluster assignments \\
        Update cluster centroids
        }
    }
    {\bf $\triangleright$ Train Classifiers} \\
    $\Zcal \leftarrow E(\Xcal)$ \\
    For every cluster $C_j$, train a classifier $f_j$ on $(\Zcal_{j}, \Ycal_{j})$.

    {\bf Output} {\bf Output} Trained Network $E$ and cluster centroids $\mu = \{\mu_j\}_{j=1}^{j=k}$ in embedded space.
\end{algorithm}

\section{Experiments}
\label{sec:expts}
\looseness=-1
We perform two sets of experiments. The first set of experiments, on synthetic data, studies the sensitivity of CAC on various characteristics of the data like cluster separation and class separation within clusters. We find that cluster-then-predict method does enhance overall classifier performance and exploiting class separation can lead to further improvements. Appendix \ref{app:synth_results_cac} presents the details of these experiments and the results.
In the second set of experiments, described below, we compare the performance of CAC with that of other methods on real benchmark datasets.

\begin{table}[!h]
\captionsetup{font=footnotesize}
\footnotesize
\centering
    \caption{Dataset Summary. MAJ: proportion of majority class.}
    \begin{tabular}    {cccc}
        \toprule
        \textbf{Dataset} & \#Instances & \#Features & MAJ\\
        \midrule
        Sepsis \cite{reyna2019early} & {$30661$} & {$209$} & $0.07$ \\
        WID Mortality \cite{lee_wids} & {$91711$} & {$241$} & $0.91$\\
        Diabetes     \cite{UCL} & {$100000$} & {$9$} & $0.41$ \\
        Respiratory \cite{johnson2016mimic} & {$22450$} & {$486$} & $0.92$ \\
        CIC \cite{johnson2012patient} & {$12000$} & {$117$} & $0.85$ \\
        CIC-LOS & {$12000$} & {$116$} & $-$ \\

        \bottomrule
    \end{tabular}
    \label{tab:datasets}
\end{table}

\paragraph{Real Datasets:}
\looseness=-1
We evaluate CAC and \deepmethod\ on 6 clinical datasets (Table \ref{tab:datasets}). CIC refers to the dataset from the 2012 Physionet challenge \cite{silva2012predicting} to predict in-hospital mortality of intensive care unit (ICU) patients at the end of their hospital stay. The dataset has time-series records, comprising various physiological parameters, of $12,000$ patient ICU stays. 
We follow the data processing scheme of \cite{johnson2012patient} (the top-ranked team in the competition) to obtain a static 117-dimensional feature vector for each patient.
The Sepsis dataset is derived from the 2019 Physionet challenge of sepsis onset prediction. It has time series records, comprising various physiological parameters of $\sim40000$ patients. We follow the data processing scheme of the competition winners \cite{morrill2019signature} to obtain static 89 dimensional features for each patient.
WID data is from the Women In Data Science challenge \cite{lee_wids} to predict patient mortality.
The Diabetes dataset is from the UCI Repository where the task is to predict patient readmission.
The Respiratory data has been extracted by us for the tasks of Acute Respiratory Distress Syndrome (ARDS) prediction. For the Sepsis and ARDS datasets, we use the first, last, median, minimum, and maximum values in the first 24 hours as features for the temporal variables.
The final feature vector dimensions for ARDS and Sepsis data are 486 and 209 respectively.
All the above tasks are posed as binary classification problems. We also derive a multiclass dataset (CIC-LOS) from the CIC dataset where we predict the Length of Stay (LOS), discretized into 3 classes (based on 3 quartiles).


\begin{table*}[htb!]
\centering
\footnotesize
\captionsetup{font=footnotesize}

\caption{AUPRC of \deepmethod\ and other baselines. Best values in \textbf{bold}. Following KM, DCN, and IDEC, we use the same local classifier (denoted by -Z) that is used by \deepmethod. AUC values, that show a similar trend, are shown in Appendix \ref{app:extended_results}.}
\label{tab:results_auprc}
\resizebox{2\columnwidth}{!}{%
\begin{tabular}{ccccccccc}
\toprule
{\textbf{Dataset}} & \textbf{k} & \textbf{KM-Z} & \textbf{DCN-Z} & \textbf{IDEC-Z} & \textbf{DMNN} & \textbf{AC-TPC} & \textbf{GRASP} & \textbf{\deepmethod} \\
 & & $-$ & (2017) & (2017) & (2020) & (2020) & (2021) & (Ours) \\

\midrule

Diabetes & $2$ & {$0.528 \pm 0.011$} & {$0.556 \pm 0.002$} & {$0.553 \pm 0.007$} & {$0.53 \pm 0.005$} & {$-$} & {$0.538 \pm 0.007$} & {$\mathbf{0.558 \pm 0.004}$} \\ 
& $3$ & {$0.514 \pm 0.016$} & {$0.556 \pm 0.003$} & {$0.548 \pm 0.01$} & {$0.539 \pm 0.007$} & {$0.523 \pm 0.0$} & {$0.542 \pm 0.007$} & {$0.56 \pm 0.002$} \\ 
& $4$ & {$0.507 \pm 0.02$} & {$0.556 \pm 0.003$} & {$0.553 \pm 0.005$} & {$0.536 \pm 0.004$} & {$0.5 \pm 0.0$} & {$0.544 \pm 0.008$} & {$0.555 \pm 0.006$} \\ 
\midrule

CIC & $2$ & {$0.568 \pm 0.014$} & {$0.609 \pm 0.036$} & {$0.674 \pm 0.036$} & {$0.636 \pm 0.03$} & {$-$} & {$\mathbf{0.757 \pm 0.036}$} & {$0.753 \pm 0.023$} \\ 
& $3$ & {$0.527 \pm 0.028$} & {$0.612 \pm 0.017$} & {$0.587 \pm 0.032$} & {$0.667 \pm 0.016$} & {$0.656 \pm 0.0$} & {$0.755 \pm 0.039$} & {$0.746 \pm 0.018$} \\ 
& $4$ & {$0.54 \pm 0.033$} & {$0.594 \pm 0.025$} & {$0.618 \pm 0.016$} & {$0.636 \pm 0.025$} & {$0.663 \pm 0.0$} & {$0.743 \pm 0.031$} & {$0.743 \pm 0.014$} \\ 
\midrule

CIC-LoS & $2$ & {$0.396 \pm 0.008$} & {$0.44 \pm 0.02$} & {$0.452 \pm 0.013$} & {$0.429 \pm 0.007$} & {$-$} & {$0.478 \pm 0.021$} & {$0.485 \pm 0.004$} \\ 
& $3$ & {$0.353 \pm 0.019$} & {$0.425 \pm 0.022$} & {$0.406 \pm 0.029$} & {$0.435 \pm 0.013$} & {$0.426 \pm 0.0$} & {$0.473 \pm 0.026$} & {$\mathbf{0.488 \pm 0.008}$} \\ 
& $4$ & {$0.344 \pm 0.018$} & {$0.443 \pm 0.025$} & {$0.437 \pm 0.014$} & {$0.429 \pm 0.007$} & {$0.402 \pm 0.0$} & {$0.474 \pm 0.018$} & {$0.487 \pm 0.006$} \\ 
\midrule




ARDS & $2$ & {$0.552 \pm 0.023$} & {$0.563 \pm 0.031$} & {$0.601 \pm 0.009$} & {$0.536 \pm 0.017$} & $-$ & {$0.493 \pm 0.004$} & {$\mathbf{0.604 \pm 0.013}$}\\ 
& $3$ & {$0.527 \pm 0.022$} & {$0.582 \pm 0.017$} & {$0.565 \pm 0.025$} & {$0.525 \pm 0.012$} & {$0.553 \pm 0.004$} & {$0.492 \pm 0.004$} & {$0.594 \pm 0.014$}\\ 
& $4$ & {$0.503 \pm 0.015$} & {$0.533 \pm 0.03$} & {$0.588 \pm 0.008$} & {$0.551 \pm 0.016$} & {$0.55 \pm 0.015$} & {$0.569 \pm 0.069$} & {$0.592 \pm 0.029$}\\ 
\midrule 

Sepsis & $2$ & {$0.532 \pm 0.014$} & {$0.595 \pm 0.023$} & {$0.597 \pm 0.021$} & {$0.539 \pm 0.003$} & $-$ & {$0.5 \pm 0.005$} & {$0.646 \pm 0.037$}\\ 
& $3$ & {$0.515 \pm 0.019$} & {$0.593 \pm 0.017$} & {$0.596 \pm 0.012$} & {$0.531 \pm 0.012$} & {$0.579 \pm 0.013$} & {$0.496 \pm 0.001$} & {$\mathbf{0.654 \pm 0.031}$}\\ 
& $4$ & {$0.517 \pm 0.012$} & {$0.553 \pm 0.029$} & {$0.601 \pm 0.011$} & {$0.531 \pm 0.018$} & {$0.574 \pm 0.02$} & {$0.499 \pm 0.002$} & {$0.628 \pm 0.02$}\\ 
\midrule 

WID-M & $2$ & {$0.645 \pm 0.026$} & {$0.71 \pm 0.029$} & {$0.69 \pm 0.061$} & {$0.597 \pm 0.013$} & {$-$} & {$0.702 \pm 0.023$} & {$\mathbf{0.723 \pm 0.015}$} \\ 
& $3$ & {$0.565 \pm 0.037$} & {$0.679 \pm 0.053$} & {$0.666 \pm 0.062$} & {$0.583 \pm 0.014$} & {$0.596 \pm 0.0$} & {$0.698 \pm 0.021$} & {$0.721 \pm 0.011$} \\ 
& $4$ & {$0.535 \pm 0.044$} & {$0.664 \pm 0.048$} & {$0.677 \pm 0.049$} & {$0.593 \pm 0.014$} & {$0.619 \pm 0.0$} & {$0.681 \pm 0.014$} & {$0.714 \pm 0.021$} \\ 

\bottomrule
\end{tabular}
}
\end{table*}

\paragraph{Experimental Setup:}
\looseness=-1
We evaluate the performance of CAC with a cluster-then-predict approach where $k$-means is used to cluster and multiple classifiers are used. We also compare its performance with that of the classifier directly, where no clustering is used. These are evaluated on a held-out test split comprising $25\%$ of the data.
\deepmethod\ and its baselines are evaluated on $5$ random train-validation-test splits for each dataset in the ratio 57-18-25 (Test data is 25\% of the entire dataset and the validation dataset is 25\% of the remaining data).
Standard binary classification metric,
the AUPRC (Area Under Precision-Recall curve) score is used to evaluate classifier performance, as there is an imbalance in many datasets. 
All models are trained until the AUPRC score on the validation data set plateaus.
The Area under the ROC Curve (AUC) 
is also reported in Appendix \ref{app:extended_results}, Table \ref{tab:results_auc}.


\paragraph{Baselines:}
\looseness=-1
For evaluating CAC, we consider $9$ classical base classifiers (X), both linear and nonlinear and report results of X, KM+X, and CAC+X. 
These are LR, Linear SVM, XGBoost (\#estimators = 10), Linear Discriminant Analysis (LDA), single Perceptron, Random Forest (\#trees = 10), $k$ Nearest Neighbors ($k=5$), SGD classifier and Ridge classifier (Ridge regressor predicting over the range $[-1,1]$).

For evaluating \deepmethod, we use two kinds of baselines.
The first set has classifiers that employ a simple cluster-then-predict kind of method
where clustering is independently performed first and classifiers are trained on each cluster (denoted by {\bf -Z}).
We compare with $3$ clustering methods \textbf{$k$-means} (where we use an autoencoder to get embeddings which are then clustered using $k$-means), Deep Clustering Network {\bf DCN} \cite{yang2017towards}, Improved Deep Embedded Clustering ({\bf IDEC}) \cite{Guo2017idec} and {\bf DMNN} \cite{li2020dmnn}.
The second set of baselines includes methods where clustering and classification happen simultaneously.
We compare \deepmethod\ with \textbf{GRASP} \cite{zhang2021grasp} modified to use the same backbone model i.e. the encoder.
We also compare with \textbf{AC-TPC} \cite{lee2020temporal} which uses an actor-critic model to refine clusters depending on the feedback from the classification process.

\paragraph{Hyperparameters:}
\looseness=-1
For DMNN and \deepmethod, the autoencoder has $3$ layers of sizes $64-32-64$ and a two-layer local predictor network of size $30-1$.
The local prediction networks use Softmax and ReLU activation functions and are not regularized. The hyperparameters for deep learning methods, KM-Z, DCN-Z ($\beta=2$), IDEC-Z ($\beta=0.5$) and \deepmethod\ ($\beta=20,\alpha=5$) are selected by conducting a sensitivity analysis (see Fig. \ref{fig:sensitivity_analysis}) on the training set of the CIC-LoS dataset. For GRASP and AC-TPC, we use the hyperparameters suggested by the authors found using sensitivity analysis. The learning rate for \deepmethod\ is set to $2e-3$. The predictor, selector and encoder networks in AC-TPC are FCNs with $64-64$, $64-64$, and $32-32$ neurons respectively.


\paragraph{Results:}
\looseness=-1
The results for CAC are presented in Table \ref{tab:cac_f1} (Appendix \ref{app:real_results_cac}) that summarises the F1 scores of CAC and baselines evaluated in clinical data sets. This is in line with Theorem 1 which predicts the better performance of multiple simple classifiers compared to a single simple classifier.
We also observe that CAC performs better than the simple $k$-means + classification approach on 20 out of 27 ($74\%$) experiments. This also serves as ablative evidence of the contribution of linear separability criteria since removing the latter reduces CAC to applying $k$-means and then predicting.

\looseness=-1
Table \ref{tab:results_auprc} presents the AUPRC scores obtained by \deepmethod\ and its baselines. 
\deepmethod\ is consistently better than the simple combined clustering and classification (cluster-then-predict) baselines i.e., DMNN, KM-Z, DCN-Z, and IDEC-Z. This shows that classification aware clustering is necessary, as done by GRASP and \deepmethod, to ensure good downstream classification performance.
GRASP and \deepmethod\ perform comparably on the CIC (moderately sized) dataset.2
\deepmethod\ outperforms GRASP by a significant margin on the rest of the datasets (the larger and more imbalanced datasets) for different number of clusters. Similar performance improvements are also observed when AUC is used as a metric (in Appendix \ref{app:extended_results}).
The code for AC-TPC (provided by the authors) did not run for $k=2$ on all datasets. Figure \ref{fig:sensitivity_analysis} and \ref{fig:cac_sil_score} show how the clustering performance of \deepmethod\ does not deteriorate with increasing epochs as compared to CAC. 

\paragraph{Case Study:}
\looseness=-1
To show the practical clinical applicability of \deepmethod\, we present a case study in Appendix \ref{app:case_study} on
length of stay prediction.
As expected, we see that inferred clusters have differences in risk factors tailored to each subpopulation.

\begin{figure}[!htb]
\begin{tabular}{cccc}
\subfloat[AUC/AUPRC vs $\alpha$ on CIC]{\includegraphics[width = 0.23\textwidth]{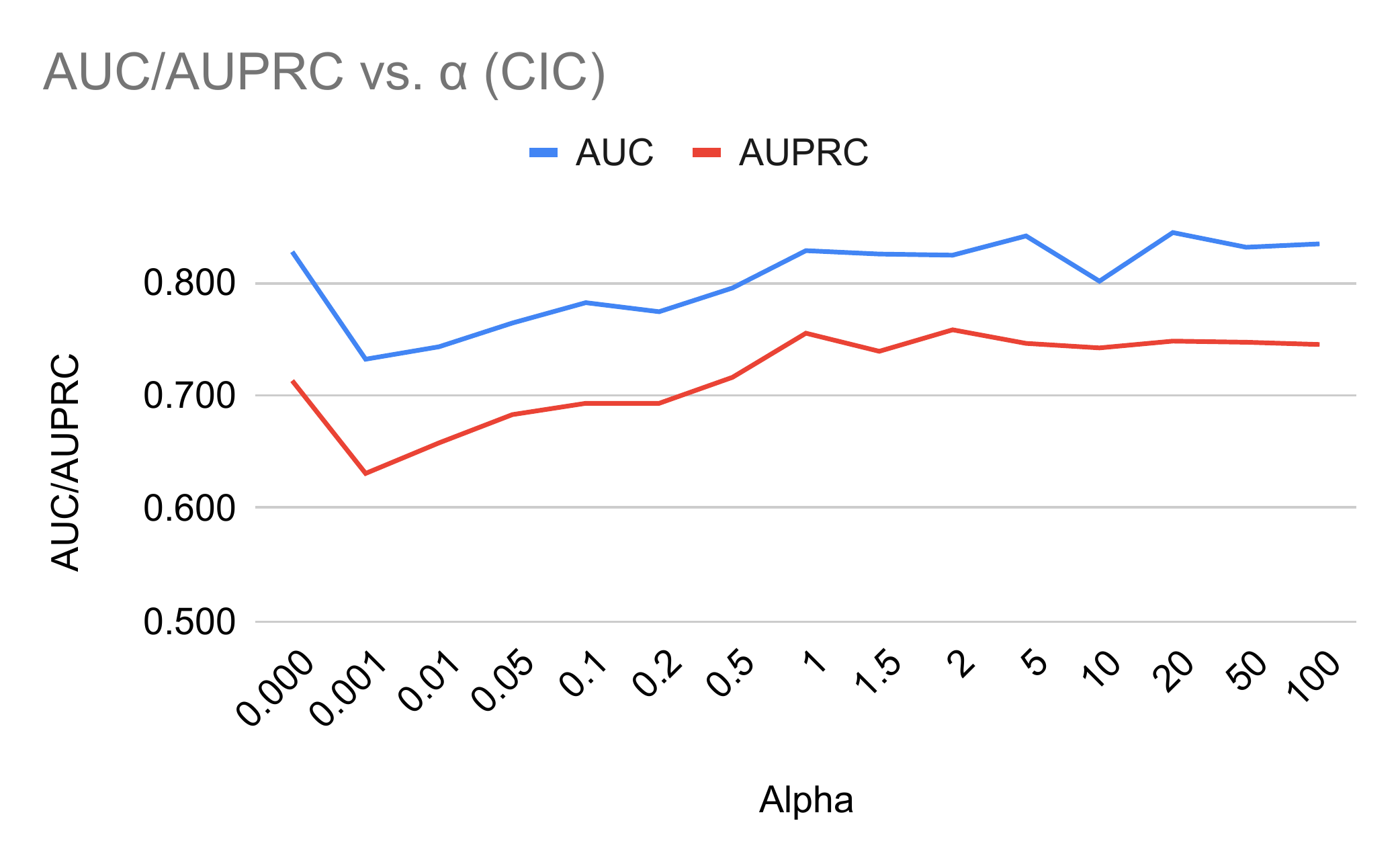}} &

\subfloat[AUC/AUPRC vs $\beta$ on CIC]{\includegraphics[width = 0.23\textwidth]{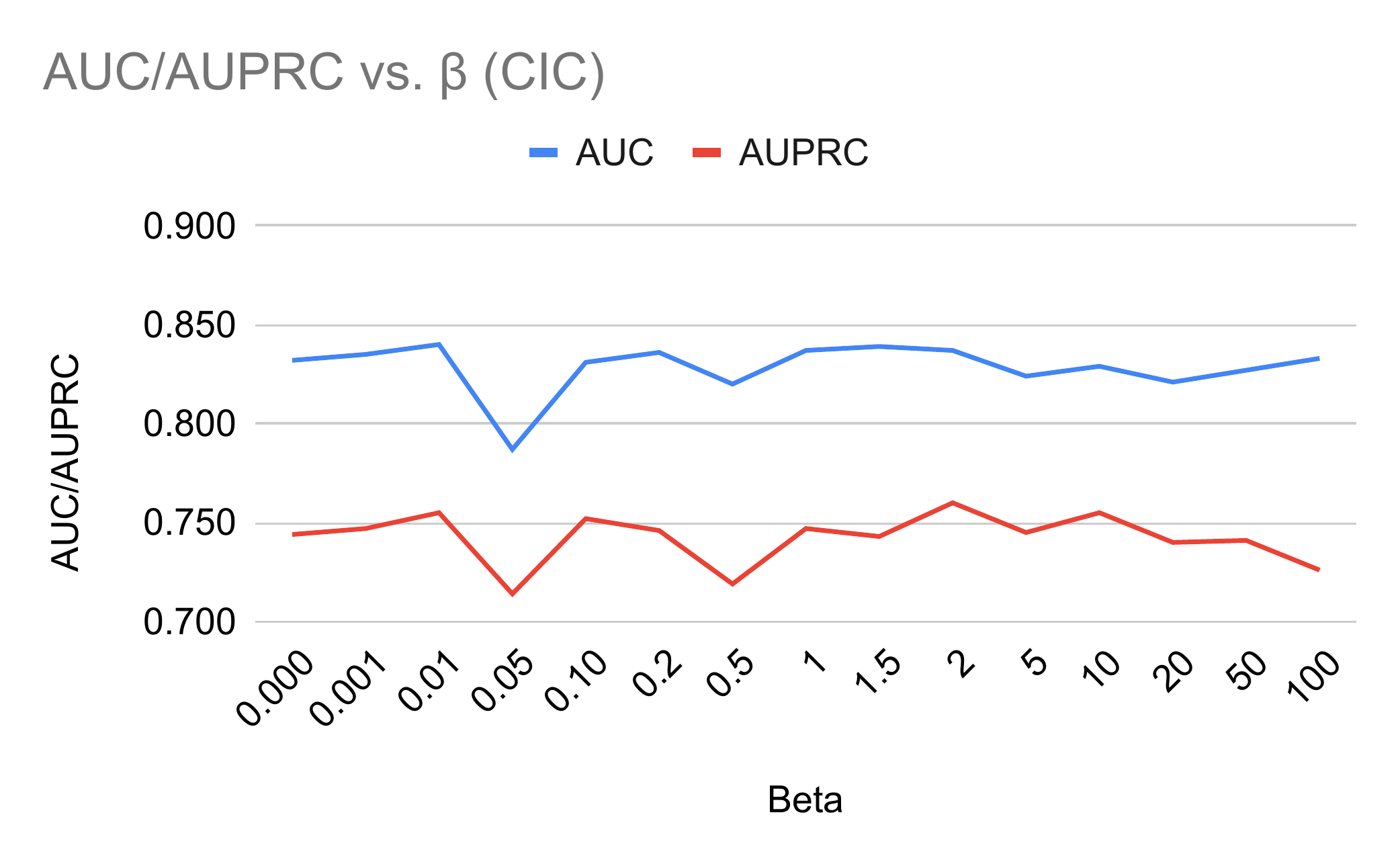}}

\end{tabular}
\caption{Sensitivity Analysis and Ablation Study for \deepmethod\ w.r.t. $\alpha$ and $\beta$ hyperparameters for $k=3$}
\label{fig:deepcac_ablation}
\end{figure}

\begin{table}[htb!]
\footnotesize
\caption{AUPRC Scores for Time series datasets. Row-wise best result in bold.
AC-TPC did not finish execution in time.}
\begin{tabular}{cccc}
\toprule
\textbf{AUPRC} & \textbf{k} & \textbf{GRASP-TS} & \textbf{\deepmethod-TS} \\
\midrule
ARDS-TS & 2 & $0.498$ & $\mathbf{0.615}$ \\
ARDS-TS & 3 & $0.506$ & $\mathbf{0.605}$ \\
ARDS-TS & 4 & $0.5$ & $\mathbf{0.565}$ \\
\midrule
Sepsis-TS & 2 & $0.501$ & $\mathbf{0.704}$ \\
Sepsis-TS & 3 & $0.507$ & $\mathbf{0.672}$ \\
Sepsis-TS & 4 & $0.495$ & $\mathbf{0.648}$ \\
\bottomrule
\end{tabular}
\label{tab:auprc_ts}
\end{table}

\paragraph{Ablation Study and Sensitivity Analysis:}
\label{sec:deepcac_ablation}
\looseness=-1
We study the effect of $\alpha$ and $\beta$ hyperparameters on the classification performance of \deepmethod\ by varying them on the range $[0.0,100]$ (see Fig. \ref{fig:deepcac_ablation}). 
Both AUC and AUPRC scores fall sharply as $\alpha$ is increased from 0 but they eventually increase, plateauing after $\alpha=5$. The classification performance does not vary much with $\beta$.
A value of 0 leads to the removal of the loss term; better results are obtained with non-zero values indicating that both loss terms contribute to performance.

\paragraph{Extending \deepmethod\ for Other Data Modalities:}
\looseness=-1
\deepmethod\ can be easily extended for other input data modalities such as images, text and time series by modifying the encoder and decoder appropriately.
We demonstrate one such case 
by using an LSTM-based AE for time series data (see Appendix \ref{app:deepcac_ts} for more details). 
We evaluate all algorithms on the ARDS and Sepsis datasets as mentioned in table \ref{tab:datasets}. We use the first 24 hours of data to predict risk (of each condition, separately) in the remaining ICU stay, which is aligned with a hospital-centric schedule for prediction.
Our experimental evaluation shows that \deepmethod\ outperforms similar extensions of GRASP. As expected, \deepmethod-TS also performs better than \deepmethod\ on the non time series versions of ARDS and Sepsis datasets.

\section{Conclusion}
\label{sec:conclusion}
\looseness=-1
In this paper, we theoretically analyze potential reasons for the good performance of classifiers trained on underlying clusters in data.
Our analysis yields insights into the benefits of using simple classifiers trained on clusters compared to simple or complex classifiers trained without clustering.
We use our analysis to develop a simple prototype algorithm, CAC, that enforces class separability within each group during cluster discovery.
Empirical and theoretical analysis of CAC further illuminates reasons for clustering aiding classification.
To further improve clustering in such settings, we combine the benefits of deep representation learning and CAC in our algorithm \deepmethod.
We use novel techniques to stabilize training and induce
class separability within clusters in \deepmethod.
Experiments on clinical benchmark datasets demonstrate the efficacy of \deepmethod\ in classification and clustering. 
A limitation of \deepmethod\ is that the hard definition of point clusters forces the local classifiers to treat interior and fringe points equally. This can lead to a loss in predictive performance and can be addressed in the future.


\bibliographystyle{unsrt}
\bibliography{References}

\newpage

\section*{Appendix}
\section{CAC details}
\looseness=-1
\label{app:cac_details}

\subsection{Class Separability}
We use a simple heuristic to measure class separability within each cluster.
For each cluster $C_j$ ($j\in [k]$), we define:
{\it cluster centroid}: $\mu(C_j) := |C_j|^{-1} \sum_{x\in C_j} x_i$,
{\it positive centroid}:  $\mu^{+}(C_j) := (\sum_{x_i \in C_j} y_i)^{-1} \sum_{x_i \in C_j} y_i x_i$, and
{\it negative centroid}: $\mu^{-}(C_j) := ({\sum_{x_i \in C_j} (1 - y_i)})^{-1} \sum_{x_i \in C_j} (1-y_i) x_i$.

Class separability is defined by considering the distance between positive and negative centroids within each cluster. Figure \ref{fig:concept} shows two clusters $C_1$ and $C_2$.
$C_1$ has greater class separation while in $C_2$, the data points are more intermingled and thus their centroids are closer.
Thus, $\|\mu_1^{+} - \mu_1^{-}\| > \|\mu_2^{+} - \mu_2^{-}\|$.

\subsection{Finding the clusters}

The most common heuristic to find $k$-means clusters is Lloyd's Algorithm \cite{lloyd1982least}.
It begins with an initial clustering and iteratively computes cluster centroids and assigns points to clusters corresponding to their closest
centroids.
However, using this heuristic with our proposed cost may not guarantee convergence.
(See our discussion after Theorem \ref{app:cac_convergence}).
Instead, we adopt an approach based on Hartigan's method \cite{wong1979algorithm}.
For $k$-means clustering, Hartigan's method proceeds point by point in a greedy manner. Each point is reassigned to a cluster such that the overall cost is reduced.

In our setting, denote the cost of merging a point $x$ to some cluster $C_q$ as $\Gamma^+(C_q, x) = \phi(C_q \cup  \{x\}) - \phi(C_q)$. Symmetrically, denote the cost of removing a point $x$ from some cluster $C_p$ as $\Gamma^-(C_p, x) = \phi(C_p\setminus \left\{x\right\}) - \phi(C_p)$.

The cost function of cluster $C_j$ with respect to point $i$  is calculated as $\|x_i - \mu_{j}\|^2 - \alpha \|\mu^+(C_j) - \mu^-(C_j)\|^2$. $\alpha$ is a hyperparameter which controls the weightage given to the class separation parameter. Typical value of $\alpha$ lies within the range $0$ to $2$ but it can vary according to the size of the dataset.

Consequently, define the improvement in cost by moving a point $x \in C_p$ from cluster $C_p$ to cluster $C_q$ to be $\Phi(x; C_p, C_q) := \Gamma^+(C_q ,x) + \Gamma^-(C_p , x)$. The exact expressions for $\Gamma^+(C_q ,x)$ and $\Gamma^-(C_p , x)$ w.r.t $x, \mu, \mu^{\pm}$ can be easily derived from above definitions and the CAC cost function (not shown here due to lack of space).

The main idea is from Hartigan's method of clustering \cite{telgarsky2010hartigan}: 
It begins with an initial clustering ($k$-means found clusters) and then proceeds iteratively in rounds.
In each round, all points are checked to determine if a point has to be re-assigned to another cluster.
For a point in the current iteration, $x_i$, if moving $x_i \in C_p$  would result in $C_p/x_i$ having points of only one class, then it is not moved.
If not, the cluster $C_q$ that results in the maximum decrease in cost $\Phi(x_i; C_p, C_q)$ is found and $x_i$ is moved from its current cluster $C_p$ to $C_q$.

In each round, all three 
centroids for each cluster are updated and once the loss function has become stagnant, the algorithm is considered to have converged. Base classifiers are then trained on the clusters found and are used for predicting labels for the unseen test data.

Classifiers $f_j$ are trained to minimize their corresponding \textit{supervised} training loss $\ell(f, C) = \sum_{j=1}^{k} \ell(f_j, C_j) $.
The best collection of classifiers, chosen based on a user-defined performance metric (e.g., training loss or validation set performance), is the final output, along with the corresponding cluster centroids.

The CAC algorithm runs until convergence and outputs the best collection of classifiers $f_{best}$ that minimizes the empirical loss.

\begin{remark*}[\bf{Hartigan vs. Lloyd}]
\label{remark:hartigan}
    For $k$-means clustering, both Lloyd's method and Hartigan's method converge since the cost function always decreases monotonically with each update.
    Note that Lloyd's method does not consider the change in centroids while updating the assignment of a point, and instead, directly assign points to their current closest centroids. 
    In our setting, Lloyd's method can similarly fix all centroids, compute a new assignment, and update the centroids accordingly.
    For instance, replacing the assignment rule of $x_i$ in Algorithm~\ref{algo:cac} by $q:=\argmin_{j} \|x_i-\mu(C_j)\|^2-\alpha\cdot\|\mu^+(C_j)-\mu^-(C_j)\|^2$.
    However, our cost function contains an additional squared distance term between positive and negative centroids, i.e., $-\alpha\cdot\|\mu^+(C_j)-\mu^-(C_j)\|^2$. 
    There is no guarantee that this term monotonically decreases by Lloyd's method during the update of $\mu^+(C_j)$ and $\mu^-(C_j)$ after an assignment.
    Hence, it is non-trivial to ensure and prove the convergence of Lloyd's method with our cost function.
    However, with the use of Hartigan's method, convergence is ensured by design. As for the rate of convergence, empirical experiments demonstrate that the rate of convergence is positively related to $\alpha$.

\end{remark*}

\subsection{CAC Prediction}
CAC outputs a set of cluster representatives denoted by their centroids and the trained classifiers associated with those clusters. 
To make predictions for a test point $\hat{x}$, the first step is to find the cluster to which the test point is most likely to belong. This can be done by assigning the test point to a cluster represented by the `closest' centroid $C_j$ in the data space i.e. $j = \argmin_{l} \|\hat{x} - \mu_l\|^2$. The corresponding classifier $f_j$ can then be used to predict $\hat{y} = f_j(\hat{x})$.

\begin{figure}[!htb]
    \centering
    \makebox[0.5\textwidth][c]{\includegraphics[width=\columnwidth]{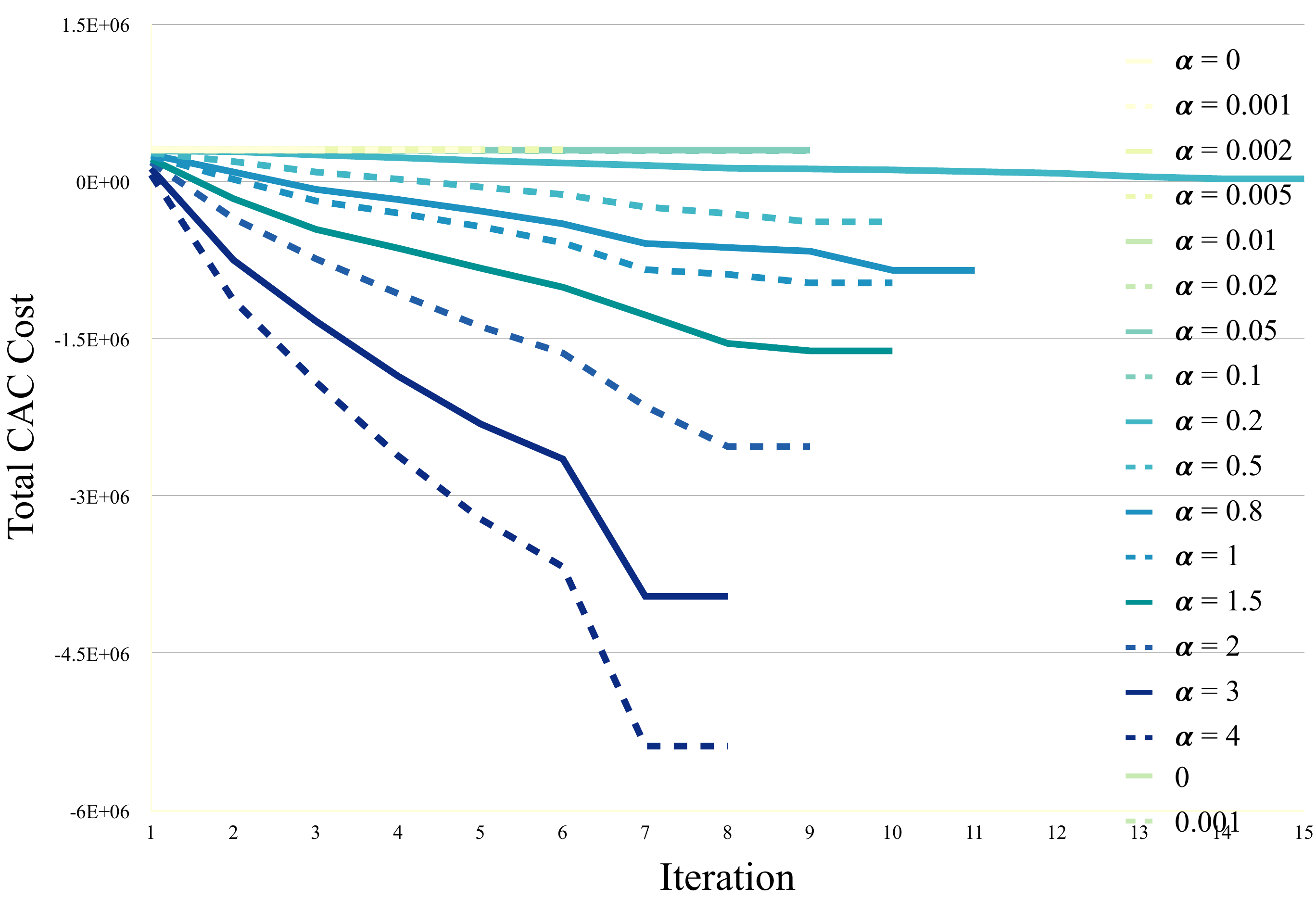}}%
    \caption{Total CAC cost v/s \# of iterations for different $\alpha$.}
    \label{fig:cac_loss}
\end{figure}

\begin{figure}[!htb]
    \centering
    \makebox[0.5\textwidth][c]{\includegraphics[width=\columnwidth]{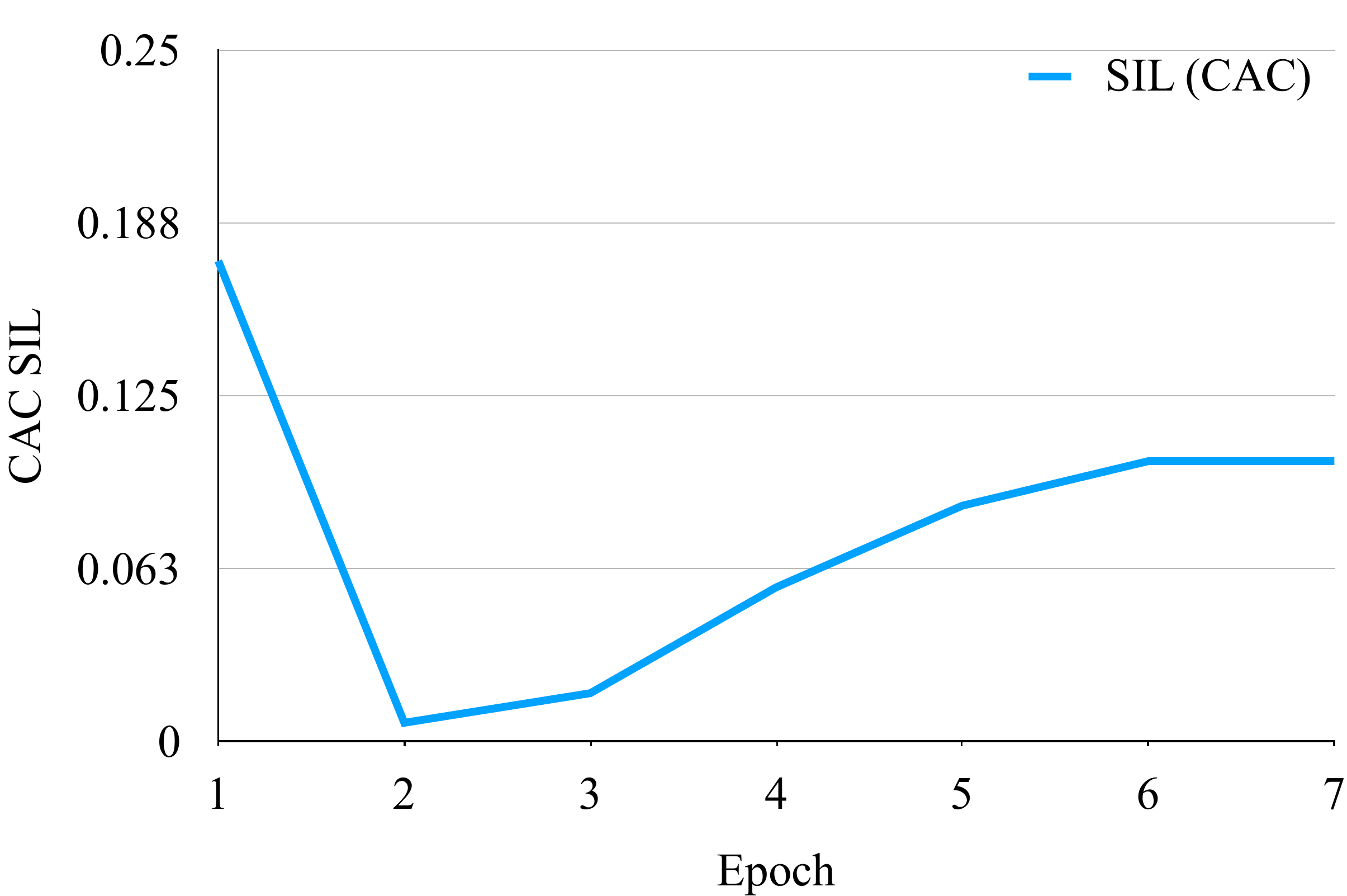}}%
    \caption{CAC clustering Silhouette score v/s training epochs}
    \label{fig:cac_sil_score}
\end{figure}

\section{Proof of Theorems}
\subsection{Proof of Theorem \ref{thm:clustering_rademacher}}
\label{app:thm_radamacher}

The following lemma is from \citep[Theorem 3.1]{bartlett2002rademacher,mohri2012foundations}: 
\begin{lemma} \label{lemma:1}
Let $\Fcal$  be a  set of maps $x\mapsto f(x)$. Suppose that $0 \le \ell\left(q, y\right) \le \lambda$ for any $q \in\{f(x):f\in \Fcal, x\in \Xcal \}$ and  $y \in \Ycal$. Then, for any $\delta>0$, with probability at least $1-\delta$ (over  an i.i.d. draw of $m$ i.i.d.  samples  $((x_i, y_i))_{i=1}^m$), the following holds: for all maps $ f\in\Fcal$,
\begin{align} 
\EE_{x,y}[\ell(f(x),y)]
& \le \frac{1}{m}\sum_{i=1}^{m} \ell(f(x_i),y_i)+2\hat \Rcal_{m}(\ell \circ \Fcal)+ \nonumber\\
& 3\lambda \sqrt{\frac{\ln(1/\delta)}{2m}},   
\end{align}
where  $\hat \Rcal_{m}(\ell \circ \Fcal):=\EE_{\sigma}[\sup_{f \in \Fcal}\frac{1}{m} \sum_{i=1}^m \sigma_i \ell(f(x_{i}),y_{i})]$ and $\sigma_1,\dots,\sigma_m$ are independent uniform random variables taking values in $\{-1,1\}$.  
\end{lemma}

We can use Lemma \ref{lemma:1} to prove the following:

\setcounter{theorem}{0}

\begin{theorem} \label{thm:}
Let $\Fcal_j$ be a  set of maps $x \in C_j \mapsto f_{j}(x)$.  Let $\Fcal= \{x\mapsto f(x):f(x) = \sum_{j=1}^K \one\{x \in C_j\} f_j(x), f_j \in \Fcal_j\}$. Suppose that $0 \le \ell\left(q, y\right) \le \lambda_{j}$ for any $q \in\{f(x):f\in \Fcal_{j}, x\in C_j\}$ and  $y \in \Ycal_{j}$.  Then, for any $\delta>0$, with probability at least $1-\delta$, the following holds: for all maps $ f \in\Fcal$,
\begin{align} 
&\EE_{x,y}[\ell(f(x),y)]
\le \sum _{j=1}^K\Pr(x \in C_j) \nonumber\\ &\left(\frac{1}{m_{j}}\sum_{i=1}^{m_{j}} \ell(f_{}(x_i^{j}),y_i^{j})+2\hat \Rcal_{m_{j}}(\ell \circ \Fcal_{j})+3\lambda_{j} \sqrt{\frac{\ln(K/\delta)}{2m_{j}}}\right),   
\end{align}
where \\
$\hat \Rcal_{m}(\ell \circ \Fcal_{j}):=\EE_{\sigma}[\sup_{f_{j} \in \Fcal_{j}}\frac{1}{m} \sum_{i=1}^m \sigma_i \ell(f_{j}(x_{i}),y_{i})]$, $Pr(x \in C_j) = 1$ if $x\in C_j$ and $\sigma_1,\dots,\sigma_m$ are independent uniform random variables taking values in $\{-1,1\}$.
\end{theorem}
\begin{proof}[Proof of Theorem \ref{thm:clustering_rademacher}]
We have the following.
\begin{align} \label{eq:1}
\EE_{x,y}[\ell(f(x),y)] &=\sum _{j=1}^K\Pr(x \in C_j)\EE_{x,y}[\ell(f(x),y) \mid x  \in C_j)]
\\ \nonumber & =\sum _{j=1}^K\Pr(x \in C_j)\EE_{x,y}[\ell(f_{j}(x),y) \mid x  \in C_j)].  
\end{align}
Since the conditional probability distribution is a probability distribution, we apply Lemma \ref{lemma:1} to each term and take a union bound to obtain the following: for any $\delta>0$, with probability at least $1-\delta$, for all $j \in \{1,\dots,K\}$ and all $f_j \in F_j$,
\begin{align*} 
&\EE_{x,y}[\ell(f_{j}(x),y) \mid x  \in C_j)]\\ &\le\frac{1}{m_{j}}\sum_{i=1}^{m_{j}} \ell(f_{j}(x_i^{j}),y_i^{j})+2\hat \Rcal_{m_{j}}(\ell \circ \Fcal_{j})+3\lambda_{j} \sqrt{\frac{\ln(K/\delta)}{2m}}
\\
& =\frac{1}{m_{j}}\sum_{i=1}^{m_{j}} \ell(f_{}(x_i^{j}),y_i^{j})+2\hat \Rcal_{m_{j}}(\ell \circ \Fcal_{j})+3\lambda_{j} \sqrt{\frac{\ln(K/\delta)}{2m}}. 
\end{align*}
Thus, using \eqref{eq:1}, we sum up both sides with factors $\Pr(x \in C_j)$ to yield:

\begin{align}
&\EE_{x,y}[\ell(f(x),y)] \nonumber\\
& =\sum _{j=1}^K\Pr(x \in C_j)\EE_{x,y}[\ell(f_{j}(x),y) \mid x  \in C_j)] \nonumber
\\
& \le \sum _{j=1}^K\Pr(x \in C_j) (\frac{1}{m_{j}}\sum_{i=1}^{m_{j}} \ell(f_{}(x_i^{j}),y_i^{j})+ \nonumber\\
& 2\hat \Rcal_{m_{j}}(\ell \circ \Fcal_{j})+3\lambda_{j} \sqrt{\frac{\ln(K/\delta)}{2m}})
\end{align}
\end{proof}

\subsection{Proof of Lemma \ref{lemma:neg_derivative}}
\label{app:lemma_neg_derivative}

\begin{lemma}
\label{lemma:neg_derivative}
    For $d$ dimensional vectors $\mathbf{u}, \mathbf{v}$ and $\boldsymbol{\beta}$, $\frac{\partial \Delta}{\partial \|\mathbf{u}\|} < 0$ for $\Delta = A - (B {\boldsymbol{\beta}} \mathbf{u} + C {\boldsymbol{\beta}} \mathbf{v})$ if $\beta_{i} u_{i} > 0, \forall i \in \{1, 2, \cdots, d\}$ and $A, B > 0$.
\end{lemma}

\begin{proof}
Let $u_i$ be the $i^{th}$ element of $\mathbf{u}$. Expanding $\frac{\partial \Delta}{\partial \|\mathbf{u}\|}$ using chain rule, we get
\begin{align*}
    \frac{\partial \Delta}{\partial \|\mathbf{u}\|} = \frac{\partial \Delta}{\partial \mathbf{u}} \cdot \frac{\partial \mathbf{u}}{\partial \|\mathbf{u}\|} =  -B\boldsymbol{\beta} \cdot \frac{\partial \mathbf{u}}{\partial \|\mathbf{u}\|}\ \  \text{(Since $\frac{\partial \Delta}{\partial \mathbf{u}} = -B\boldsymbol{\beta} \mathbf{u}$)}
\end{align*}

Let $u_i$ denote the $i^{th}$ element of $\mathbf{u}$. Then, from the definition of vector derivatives we have,
\begin{align*}
\frac{\partial \mathbf{u}}{\partial \|\mathbf{u}\|} &= \left[\frac{\partial u_1}{\partial \|\mathbf{u}\|}, \frac{\partial u_2}{\partial \|\mathbf{u}\|}, \cdots, \frac{\partial u_d}{\partial \|\mathbf{u}\|} \right] \\
&= \left[\left(\frac{\partial \|\mathbf{u}\|}{\partial u_1} \right)^{-1}, \left(\frac{\partial \|\mathbf{u}\|}{\partial u_2} \right)^{-1}, \cdots, \left(\frac{\partial \|\mathbf{u}\|}{\partial u_d} \right)^{-1} \right] \\
&=\|\mathbf{u}\| \left[\frac{1}{u_1}, \frac{1}{u_2}, \cdots, \frac{1}{u_d} \right]
\end{align*}

Thus, $\frac{\partial \Delta}{\partial \|\mathbf{u}\|} = -B \cdot \|\mathbf{u}\| \cdot \mathlarger{\sum}_{i=1}^{d} \frac{\beta_i}{u_i}$.
where $u_i = \mu^{+}_i - \mu^{-}_i$ and $v_i = \mu^{+}_i + \mu^{-}_i$.

Since $\beta\cdot \mu^{+} > 0$ and $\beta\cdot \mu^{-} < 0$, it implies that $\beta \cdot \mathbf{u} > 0$. Assuming that every dimension of $\mathbf{u}$ lies on the same side of regression plane $\beta$ as the original point, i.e. $\beta_i u_i = \beta_{i} > 0, \forall i$, $\frac{\beta_i}{u_i} > 0\ \forall i$, then $\frac{\partial \Delta}{\partial \|\mathbf{u}\|} < 0$ as $B$ and $\|\mathbf{u}\|$. Thus under above assumption, we prove that $\Delta$ decreases as $\|\mathbf{u}\|$ increases.

\end{proof}

\subsection{Proof of Theorem \ref{thm:bounding_log_loss}}
\label{app:bounding_log_loss}

\begin{theorem}
Let $X = (x_1, x_2, \cdots, x_N)$ be a training dataset with $n$ records $x_i \in \mathbb{R}^{d}$ together with binary labels $y_i\in \left\{0,1\right\}$. Define the log-loss over the entire dataset as :
 $$\ell(X) = -\Sigma_{i} y_i \ln(p(x_i)) + (1-y_i) \ln(1-p(x_i))$$

where $p(x_i) = [1 + \exp(-\beta x_i)]^{-1}$ and $\beta = \argmin_{\beta} \ell(X)$. Then,
\[
C_1 - (C_3 \beta\mu^{+} - C_4\beta\mu^{-}) \leq \ell(X) \leq C_2 - (C_3\beta\mu^{+} - C_4\beta\mu^{-})
\]
where $C_1 = N\ln(2),\ C_2 = N\ln(1+\exp(c)) - \frac{Nc}{2},\ C_3 = \frac{N^{+}}{2},\ C_4 = \frac{N^{-}}{2},\ N^{+} = \sum_{x_i \in X} y_i$, $N^{-} = \sum_{x_i \in X} (1-y_i)$ and $c = \argmax_{i} \|\beta x_i\|$.
\end{theorem}

\begin{proof}
%
\textbf{Lower Bound}
Let $\ell(X, y) = \Sigma_{y_i = 1} \ell^{+} (x_i) + \Sigma_{y_i = 0} \ell^{-}(x_i)$ s.t.
\begin{align}
\label{eq:log_loss}
\ell^{+}(x_i) = -\ln(p_i),\ \ell^{-}(x_i) = -\ln(1-p_i)
\end{align}

Applying Jensen's inequality on $\ell^{+}$ and $\ell^{-}$, 
\begin{align*}
    \ell(X) =& \Sigma_{x_i \in X^{+}} \ell^{+}(x_i) + \Sigma_{x_i \in X^{-}} \ell^{-}(x_i)\\
     =& -\Sigma_{y_i = 1} \ln[1+\exp(-\beta x_i)]^{-1} - \\
     & \Sigma_{y_i = 0} \ln[1-(1+\exp(-\beta x_i))^{-1}] && \text{(By eq. \ref{eq:log_loss})}
\end{align*}

\begin{align*}
    \geq& - N^{+} \ln\left[1+\exp\left(-\beta \cdot \left(\frac{\Sigma_{x_i \in X^{+}}x_i}{N^{+}}\right)\right)\right]^{-1} \\
    & - N^{-} \ln\left[1-\left(1+\exp\left(-\beta \cdot \left(\frac{\Sigma_{x_i \in X^{-}}x_i}{N^{-}}\right)\right)\right)^{-1}\right]\\
    & \text{(Jensen's inequality for convex functions)} \\
    \geq& N^{+} \ln[1+ \exp(-\beta \mu^{+})] - N^{-} \ln\left[\frac{\exp(-\beta \mu^{-})}{1 + \exp(-\beta \mu^{-})}\right] \\ & \text{(Substituting $\mu^{+}$ and $\mu^{-}$)} \\
    \geq& N^{+} \ln[1+ \exp(-\beta \mu^{+})] + N^{-} \ln[1+ \exp(\beta \mu^{-})]
\end{align*}


We observe that $\ln(1+\exp(x)) \geq \ln(2) + \frac{x}{2} \ \forall x\in \RR$ by considering the Taylor series expansion of $\ln(1+\exp(x))$ at $x=0$. $\ln(1+\exp(x)) = \ln(2) + \frac{x}{2} + \frac{x^2}{8} + \mathcal{O}(x^4)$. Then,
\begin{align*}
\ell(X) &\geq N^{+}\left(\ln(2) -\frac{\beta\mu^{+}}{2}\right) + N^{-}\left(\ln(2) + \frac{\beta\mu^{-}}{2}\right)\\
&\geq N \ln(2) - \frac{1}{2}\cdot (N^{+}\beta \mu^{+} - N^{-} \beta \mu^{-})
\end{align*}
where $C_1 = N\ln(2), C_3 = \frac{N^{+}}{2}, C_4 = \frac{N^{-}}{2}$.

\textbf{Upper Bound}

Since all points $x_i$ are constrained within their respective clusters, we note that the maximum log-loss for a point $x_i$ labelled as $y_i$ is not unbounded. Let $c = \argmax_i\|\beta x_{i}\|$. Since $\ell^{+}$ is a monotonically decreasing function and $\ell^{-}$ is monotonically increasing, we can bound them using linear functions $s^{+}$ and $s^{-}$ respectively. $s^{+} \geq \ell^{+}$ and $s^{-} \geq \ell^{-}$ for all points $X^{+}$ and $X^{-}$ respectively. $s^{+} (x_i)$ passes through $(-c, \ln(p(-c)))$ and $(c, \ln(p(c)))$. Hence $s^{+}(x) = K_1 - K_2\cdot \beta x$ where $K_1, K_2 > 0$. Similarly, $s^{-}(x) = K_3 + K_4\cdot \beta x$ where $K_3, K_4 > 0$.


Solving for $K_1 \cdots K_4$, we get
\begin{align*}
    K_1 = K_3 = \ln(1+\exp(c)) - \frac{c}{2},\ K_2 = K_4 = \frac{1}{2}
\end{align*}
Then,
\begin{align*}
    \ell(X) &= \Sigma_{x_i \in X^{+}} \ell^{+}(x_i) + \Sigma_{x_i \in X^{-}} \ell^{-}(x_i)\\
    &< \Sigma_{x_i \in X^{+}}s^{+}(x_i) + \Sigma_{x_i \in X^{-}} s^{-}(x_i)\\
    &\leq N^{+} s^{+}\left(\frac{\Sigma_{x_i \in X^{+}}x_i}{N^{+}}\right) + N^{-} s^{-}\left(\frac{\Sigma_{x_i \in X^{-}}x_i}{N^{-}}\right) \\
    & \text{(Jensen's inequality for a linear function)}\\
    &\leq N^{+} s^{+}(\mu^{+}) + N^{-} s^{-}(\mu^{-}) \\
    &\leq N^{+}(K_1 - K_2\cdot \beta \mu^{+}) + N^{-}(K_3 + K_4\cdot \beta \mu^{-})\\
    &\leq C_2 - (C_3 \beta \mu^{+} - C_4 \beta \mu^{-}) \ \text{(Rearranging terms)}
\end{align*}


where $C_2 = N\ln(1+\exp(c)) - \frac{Nc}{2}$ and $C_3 = \frac{N^{+}}{2}, C_4 = \frac{N^{-}}{2}$.
\end{proof}

\subsection{Proof of Theorem \ref{thm:ams_cac}}
\label{app:bounding_ams_loss}

\begin{theorem}
Let $X = (x_1, x_2, \cdots, x_N)$ be a training dataset with $n$ records $x_i \in \mathbb{R}^{d}$ together with binary labels $y_i\in \left\{0,1\right\}$. Define the AM Softmax loss over the entire dataset as : $-\sum_{i}^{} \log \frac{e^{s \cdot\left(W_{y_{i}}^{T} z_{i}-m\right)}}{e^{s \cdot\left(W_{y_{i}}^{T} z_{i}-m\right)}+\sum_{j=1, j \neq y_{i}}^{c} e^{s W_{j}^{T} z_{i}}}$. Then,
\[
K_1 - \Gamma(K_3 \mu^{+} - K_4\mu^{-}) \leq \ell(X) \leq K_2 - \Gamma(K_3\mu^{+} - K_4\mu^{-})
\]
where $K_1 = N\left(\log2 + \frac{sm}{2}\right),\ K_2 = N(1 + sm),\ K_3 = sN^{+},\ K_4 = sN^{-},\ N^{+} = \sum_{x_i \in X} y_i$, $N^{-} = \sum_{x_i \in X} (1-y_i)$ and $\Gamma = W^T_{+} - W^T_{-}$.
\end{theorem}

\begin{proof}
Rewriting $L_{AM}$ by collecting all the terms for the positive and negative classes, we get
\begin{align*}
    L_{AM} = -&\sum_{i:y_i=0}^{} \log \frac{e^{s \cdot\left(W_{-}^{T} z_{i}-m\right)}}{e^{s \cdot\left(W_{-}^{T} z_{i}-m\right)}+ e^{s W_{+}^{T} x_{i}}} + \\
    -&\sum_{i:y_i=1}^{} \log \frac{e^{s \cdot\left(W_{+}^{T} z_{i}-m\right)}}{e^{s \cdot\left(W_{+}^{T} z_{i}-m\right)}+ e^{s W_{-}^{T} x_{i}}} \\
    = &\sum_{i:y_i=0}^{} \log 1+{e^{sm+sz_i(W_{+}-W_{-})}} + \\
    &\sum_{i:y_i=1}^{} \log 1+{e^{sm+sz_i(W_{-}-W_{+})}}
\end{align*}

\paragraph{Lower Bound} Let $g^+ = \log [1+{e^{sm+sz_i(W_{-}-W_{+})}}]$ and $g^- = \log [1+{e^{sm+sz_i(W_{+}-W_{-})}}]$. Note that $g^{+}$ and $g^{-}$ are convex functions. Applying Jensen's inequality on $g^{+}$ and $g^{-}$ we get:

\begin{align*}
    L_{AM} = &\sum_{i:y_i=0}^{} \log [1+{e^{sm+sz_i(W_{+}-W_{-})}}] + \\
    &\sum_{i:y_i=1}^{} \log[1+{e^{sm+sz_i(W_{-}-W_{+})}}]\\
    &> N^{+}\log[1+{e^{sm+sz\mu^{+}(W_{+}-W_{-})}}] +\\
    & N^{-}\log[1+{e^{sm+sz\mu^{-}(W_{-}-W_{+})}}]
\end{align*}

Since $\log(1+e^x) > \log2 + \frac{x}{2}\ \forall x$, we can further simplify the above expression.

\begin{align*}
    L_{AM} > &N^{-}\log[1+{e^{sm+sz\mu^{+}(W_{+}-W_{-})}}] +\\
    & N^{+}\log[1+{e^{sm+sz\mu^{-}(W_{-}-W_{+})}}] \\
    & > N^{-}[\log2 + \frac{1}{2}(sm + s\mu^{-}(W_{+}-W_{-}))] + \\
    & N^{+}[\log2 + \frac{1}{2}(sm + s\mu^{+}(W_{-}-W_{+}))] \\
    & > N\left(\log2 + \frac{sm}{2}\right) - s(W_{+}-W_{-})[N^{+}\mu^{+} - N^{-}\mu^{-}]
\end{align*}

\paragraph{Upper Bound} For all $x>0$, $\log(1+ e^x) < 1 + x$. Since $sm+sz_i(W_{+}-W_{-}) > 0$ for all $z_i$ s.t. $y_i=0$ and $sm+sz_i(W_{-}-W_{+}) > 0$ for all $z_i$ s.t. $y_i=1$ (see \cite{wang2018additive}), we can safely apply the above mentioned inequality to get an upper bound for $L_{AM}$.

\begin{align*}
    L_{AM} = &\sum_{i:y_i=0}^{} \log [1+{e^{sm+sz_i(W_{+}-W_{-})}}] + \\
    &\sum_{i:y_i=1}^{} \log[1+{e^{sm+sz_i(W_{-}-W_{+})}}] \\
    L_{AM} < &N^{+}[1+sm+sz_i(W_{-}-W_{+})] + \\ &N^{-}[1+sm+sz_i(W_{-}-W_{+})] \\
    < &(N^{+} + N^{-})(1+sm) + sN^{-}\mu^{-}(W_{+}-W_{-}) + \\
    & sN^{+}\mu^{+}(W_{-}-W_{+}) \\
    < &N(1+sm) - s(W_{+}-W_{-})(N^{+}\mu^{+} - N^{-}\mu^{-})
\end{align*}

\end{proof}



\subsection{Theorem on Convergence of CAC}
\label{app:cac_convergence}
\setcounter{theorem}{1}

\begin{theorem}
Algorithm~\ref{algo:cac} converges to a local minimum.
\end{theorem}

\begin{proof}
To prove convergence, it suffices to verify that the cost function, that re-assigns point $x$ from cluster $C_p$ to cluster $C_q$, monotonically decreases at each iteration.
This is equivalent to showing
$
\phi(C_p\setminus\left\{x_i\right\}) + \phi(C_q\cup \left\{x_i\right\}) < \phi(C_p) + \phi(C_q) 
\implies (\phi(C_p\setminus\left\{x_i\right\}) - \phi(C_p))  + (\phi(C_q\cup \left\{x_i\right\}) - \phi(C_q)) < 0 
\implies \Gamma^+(C_q ,x) + \Gamma^-(C_p , x) < 0 
\implies \Phi(x_i, C_p, C_q) < 0.
$

$\Phi(x_i, C_p, C_q) < 0$ is exactly the update condition for the algorithm that ensures that the total cost function is always decreasing with every point update and in every iteration.
If we move $\epsilon<1$ away from our binary cluster indicator then we do not get a binary result anymore, which is not in the feasible set and which is hence not minimizing the objective. Hence the clustering derived is trivially a local minimum.
\end{proof}

\begin{figure*}[!htb]
\begin{tabular}{cccc}
\subfloat[Effect of Inter Class Separation (ICS) on F1 and AUC.]{\includegraphics[width = 0.24\textwidth]{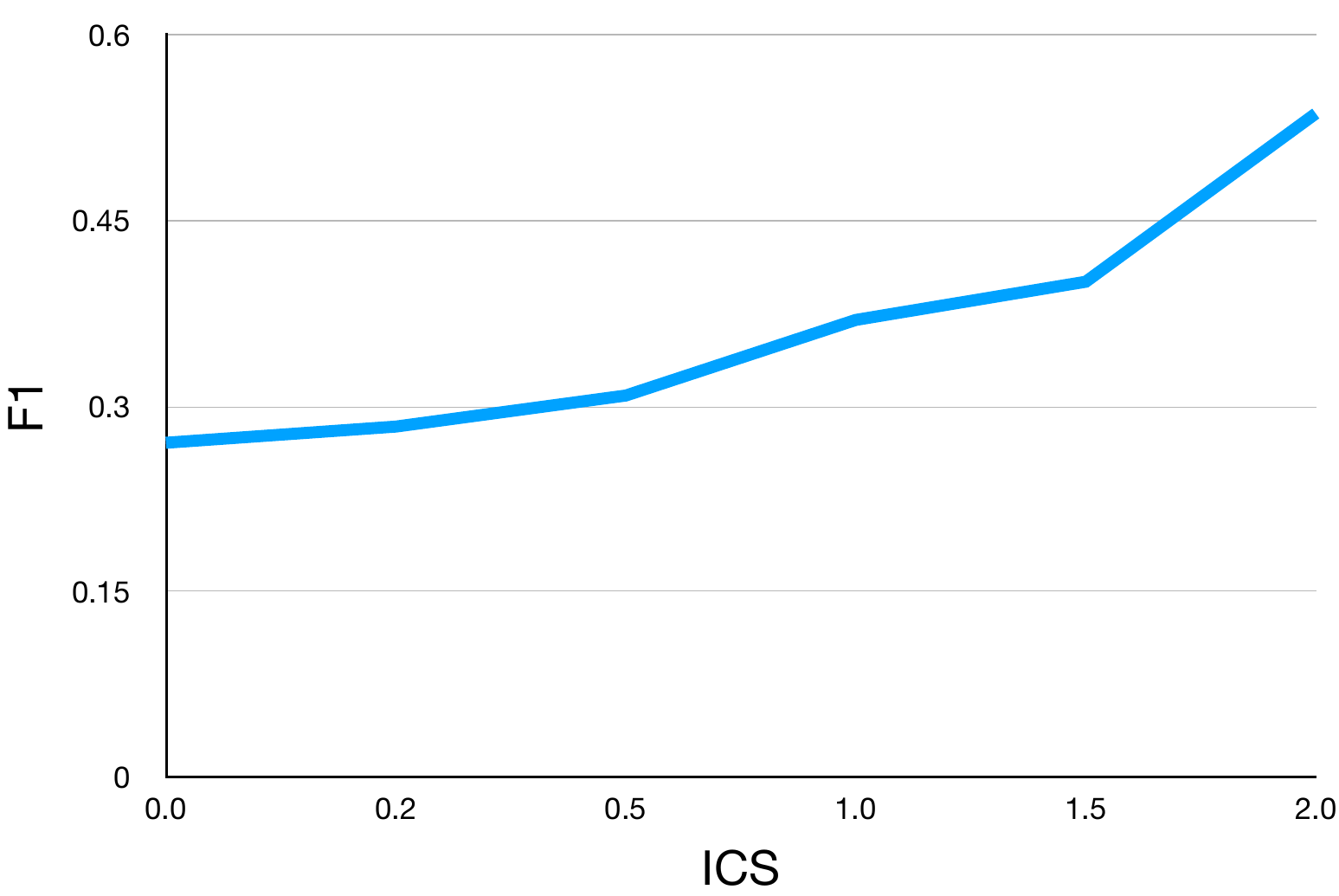}} &

\subfloat[Effect of Outer Cluster Separation (OCS) on F1 and AUC.]{\includegraphics[width = 0.24\textwidth]{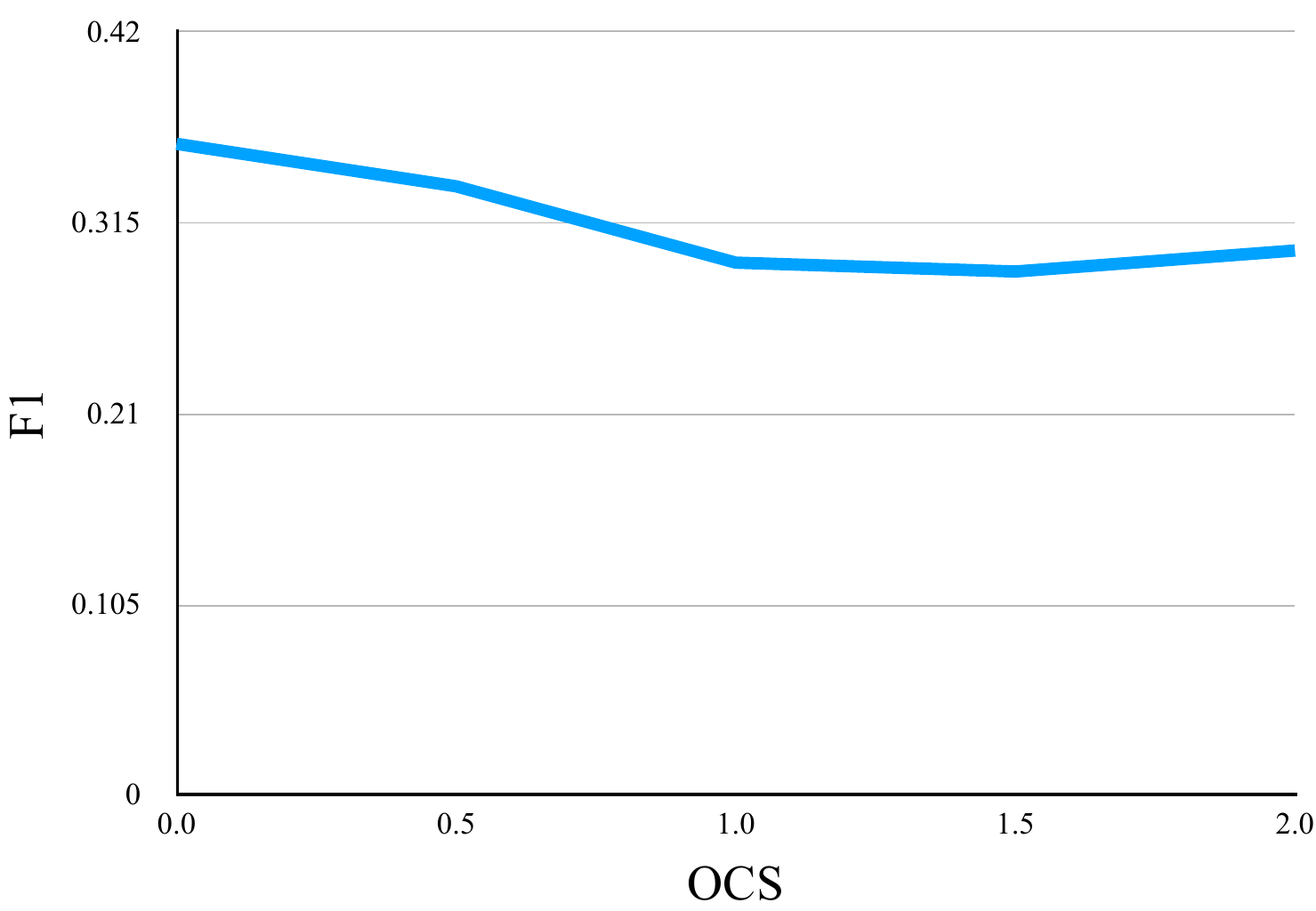}} &

\subfloat[Effect of number of natural clusters ($K$) on F1 and AUC.]{\includegraphics[width = 0.24\textwidth]{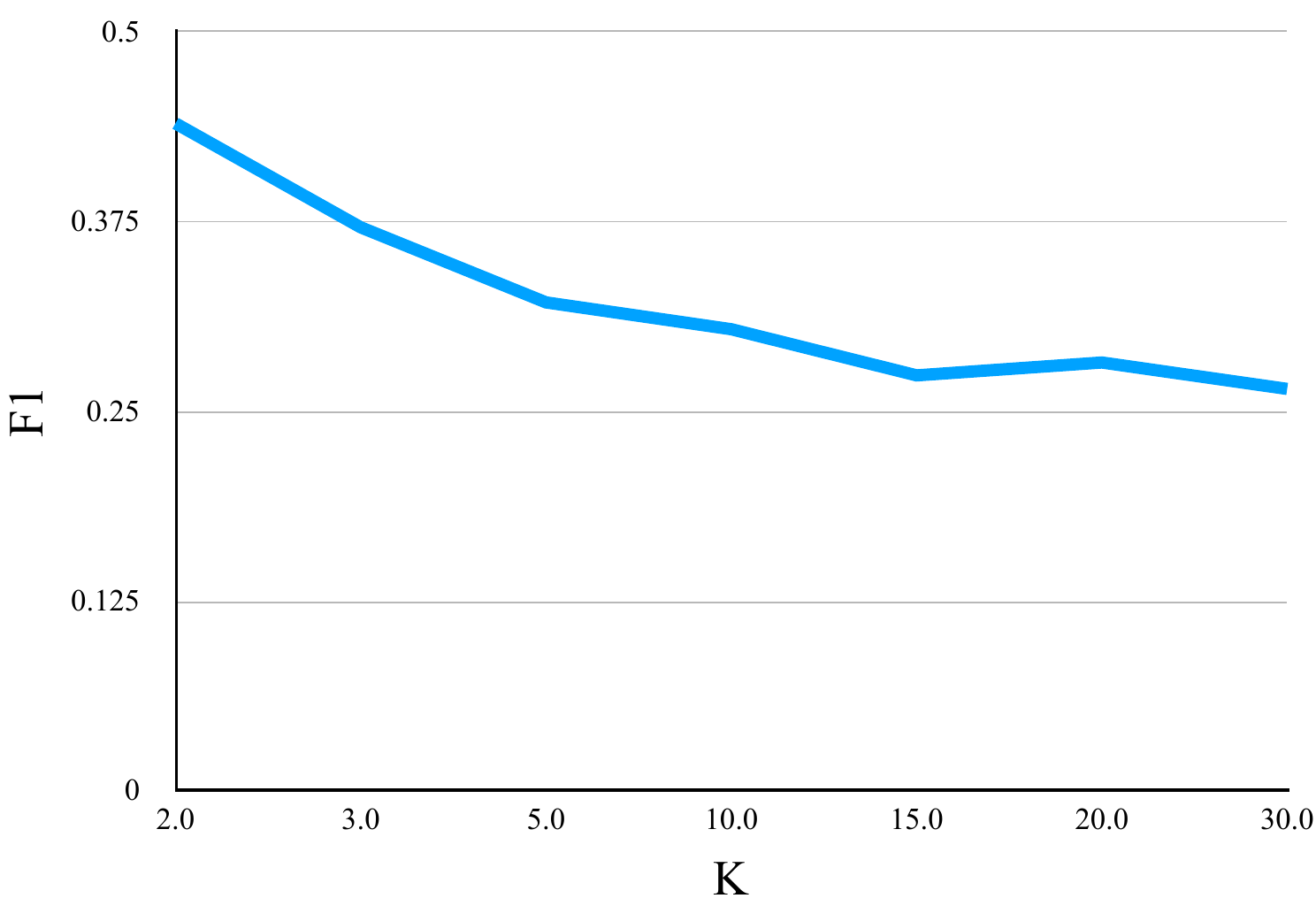}} &
\subfloat[AUC with varying \#classifiers ($k$) and Natural \#Clusters ($K$).]{\includegraphics[width = 0.24\textwidth]{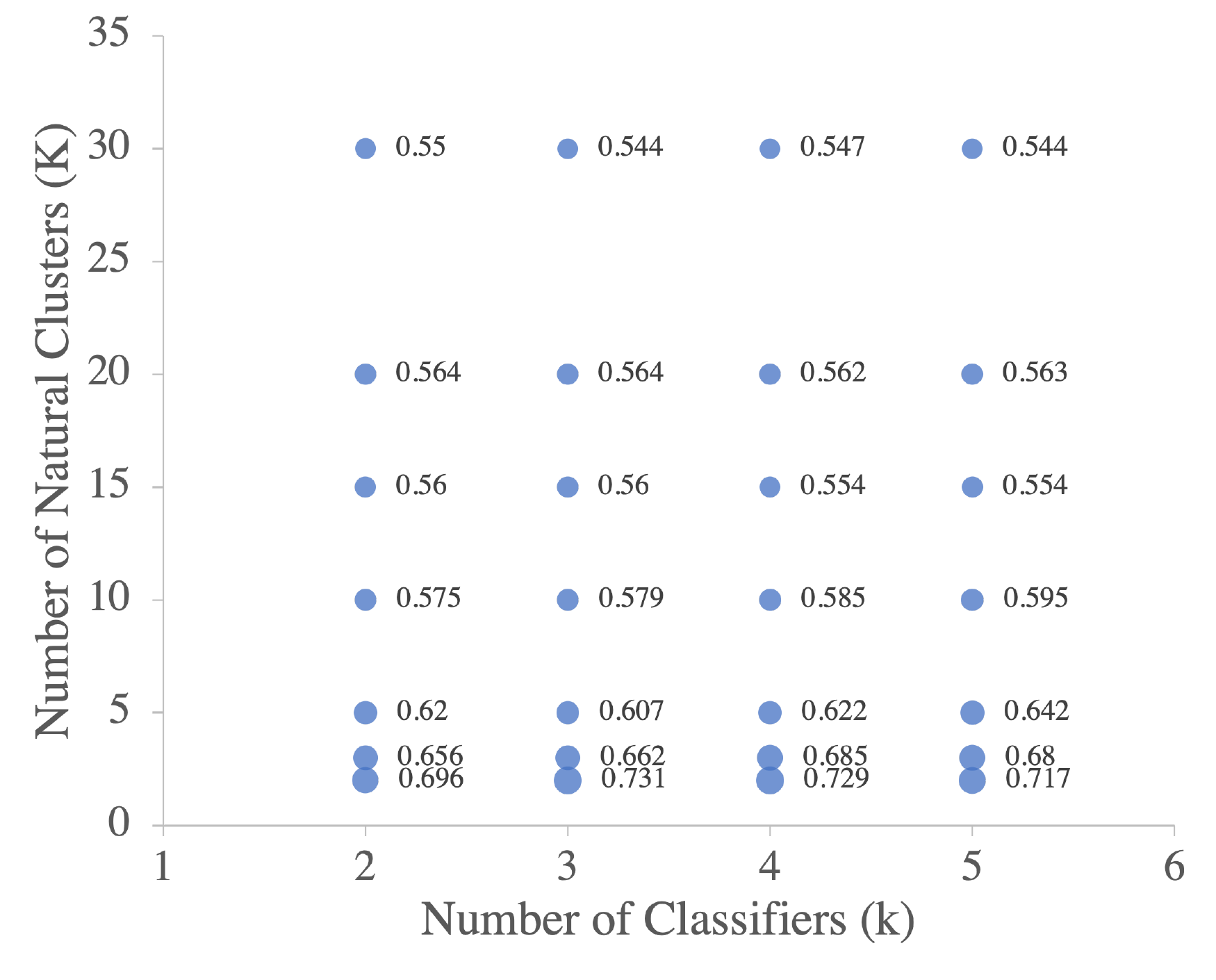}}
\end{tabular}
\caption{Evaluating the effect of different parameters on the performance of CAC}
\label{fig:cac_simulations}
\end{figure*}

\subsection{Theorem on Time Complexity of CAC}
\label{app:cac_time_complexity}

\begin{theorem}
The time complexity for each round of CAC is $O(ndk)$.
\end{theorem}
\begin{proof}
For the time complexity of each round, it suffices to prove the running time for each point $x_i$ is $O(dk)$. 
%
For each $j\in [k]$, we denote the centroid, positive centroid and negative centroid of $C_j\cup \left\{x_i\right\}$ as $\tilde{\mu}(C_j)$, $\tilde{\mu}^+(C_j)$ and $\tilde{\mu}^-(C_j)$ respectively.
Given a cluster $C_j$ and a point $z\in \mathbb{R}^d$, we denote the clustering cost of $C$ w.r.t. $z$ to be
\[
g(C_j,z) := \sum_{x\in C_j} \|x-z\|^2.
\]

A standard result from the $k$-means literature \cite{dasgupta2008hardness} is the following bias-variance decomposition of the $k$-means cost function:
\begin{equation*}
    g(C_j, z) = \phi(C_j,\mu(C_j)) + |C_j| \cdot \|\mu(C_j) - z\|^2.
\end{equation*}
Note that the above result holds for CAC cluster cost function also as can be shown easily by a little algebra. Based on this property, we have

\begin{align*}
    &\Gamma^+(C_j,x_i) \\
    = & \|\tilde{\mu} - x_i\|^2 + |C_j|\cdot \|\tilde{\mu}(C_j) - \mu(C_j)\|^2 \\
    & +\alpha |C_j|\cdot \|\mu^{+}(C_j)- \mu^{-}(C_j)\|^2 \\
    &- \alpha (|C_j|+1)\cdot \|\tilde{\mu}^{+}(C_j) - \tilde{\mu}^{-}(C_j)\|^2 \\
    &\Gamma^-(C_j, x_i) \\
    = & -\|\mu(C_j) - x_i\|^2 - (|C_j|-1) \|\tilde{\mu}(C_j) - \mu(C_j)\|^2 \\
    & + \alpha|C_j|\cdot \|\mu^{+}(C_j) -\mu^{-}(C_j)\|^2 \\
    & - \alpha (|C_j|-1)\cdot \|\tilde{\mu}^{+}(C_j) - \tilde{\mu}^{-}(C_j)\|^2
\end{align*}

Then it suffices to prove that both $\Gamma^+(C_j,x_i)$ and $\Gamma^-(C_j, x_i)$ can be computed in $O(d)$ time, which implies $O(dk)$ time for computing $q:= \argmin_{j} {\Phi(x_i; C_p, C_j)}$.
Moreover, by the above formulations, it suffices to prove that $\tilde{\mu}(C_j)$, $\tilde{\mu}^+(C_j)$ and $\tilde{\mu}^-(C_j)$ can be computed in $O(d)$ time.
Note that $\tilde{\mu}(C_p) = \frac{|C_p|\cdot \mu(C_p) - x_i}{|C_p|-1}$ and $\tilde{\mu}(C_j) = \frac{|C_j|\cdot \mu(C_j) + x_i}{|C_j|+1}$ for $j\neq p$.
Hence, each $\tilde{\mu}(C_j)$ can be computed in $O(d)$ time.
We discuss the following two cases for the computation of $\tilde{\mu}^+(C_j)$ and $\tilde{\mu}^-(C_j)$.

\begin{itemize}
        \item Case $y_i = 1$. We have that $\tilde{\mu}^+(C_p) = \frac{\left(\sum_{x_l \in \mathcal{C}_p} y_l \right)\cdot \mu_p^+ - x_i}{\sum_{x_l \in \mathcal{C}_p} y_l - 1}$ and $\tilde{\mu}^-(C_p) = \mu^-(C_p)$. 
        For $j\neq p$, we have that $\tilde{\mu}^+(C_j) \leftarrow \frac{\left(\sum_{x_l \in \mathcal{C}_j} y_l\right)\cdot \mu_j^+ + x_i}{\sum_{x_l \in \mathcal{C}_j} y_l + 1}$ and $\tilde{\mu}^-(C_j) = \mu^-(C_j)$. 
        \item Case $y_i = 0$. We have that $\tilde{\mu}_p^- \leftarrow \frac{\left(\sum_{x_l \in \mathcal{C}_p} (1-y_l)\right)\cdot \mu_p^- - x_i}{\sum_{x_l \in \mathcal{C}_p} (1-y_l) - 1}$ and $\tilde{\mu}^+(C_p) = \mu^+(C_p)$. 
        For $j\neq p$, we have that $\tilde{\mu}_j^- \leftarrow \frac{\left(\sum_{x_l \in \mathcal{C}_j} (1-y_l)\right)\cdot \mu_j^- + x_i}{\sum_{x_l \in \mathcal{C}_j} (1-y_l) + 1}$ and $\tilde{\mu}^+(C_j) = \mu^+(C_j)$.  
    \end{itemize}
    
By the above update rules, both $\tilde{\mu}^+(C_j)$ and $\tilde{\mu}^-(C_j)$ can be computed in $O(d)$ time, which completes the proof.
\end{proof}

\subsection{Hyperparameter Details}
\label{app:cac_hyperparameters}

CAC's $\alpha$ values are found by performing a $5$-fold CV Grid search (see App. \ref{app:cac_hyperparameters}). Table \ref{tab:alpha} shows the tuned values of $\alpha$ used by different variants of CAC on different datasets. We find these values by performing a $5$-fold Grid search cross validation. Figure \ref{fig:cac_loss} shows how the CAC cost function decreases monotonically (as proved in Theorem \ref{app:cac_convergence}) over iterations for different values of $\alpha$.

\begin{figure*}[t]
    \centering
    \makebox[\textwidth][c]{\includegraphics[width=\linewidth]{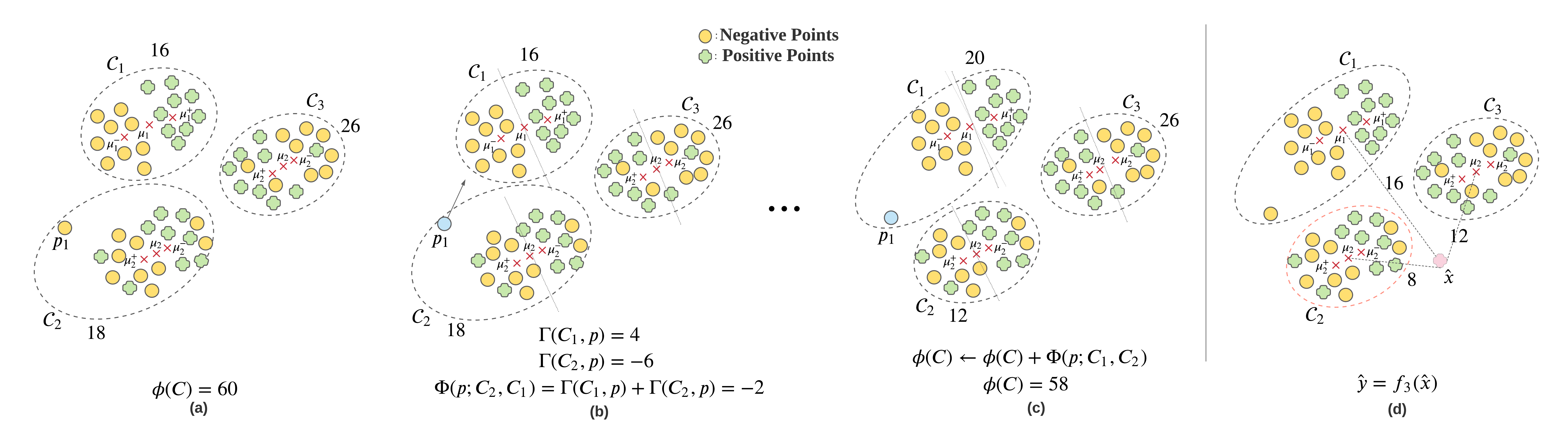}}%
    \caption{CAC problem setting (with dummy values for illustrative purposes): (a) The total CAC cost of this clustering is $60$; (b) Point $p_1$ is selected to be reassigned to cluster $C_1$ based on the cluster update equations; (c) $p_1$ is assigned to $C_1$ and the cost functions of $C_1$ and $C_2$ are updated; (d) At testing time, $\hat{x}$ is assigned to cluster $C_2$ as it lies nearest to $\hat{x}$.}
    \label{fig:concept}
\end{figure*}

\begin{table*}[!htb]
\centering
\caption{Tuned $\alpha$ values used for CAC in experiments found after 5-fold GridCV search}
\begin{tabular}{llllllllll}
\toprule
Classifier & LR & SVM & LDA & Perceptron & RF & KNN & SGD & Ridge & XGB \\
\midrule
Diabetes & 2.5 & 2.5 & 2.5 & 2.5 & 2.5 & 2.5 & 2.5 & 2.5 & 3 \\
CIC & 0.05 & 0.5 & 0.05 & 0.5 & 0.5 & 0.5 & 0.01 & 0.5 & 0.01 \\
WID-M & 0.05 & 0.5 & 0.05 & 0.5 & 0.5 & 0.5 & 0.01 & 0.5 & 0.01 \\
\bottomrule
\end{tabular}
\label{tab:alpha}
\end{table*}

\begin{table}[!htb]
\centering
\footnotesize
\captionsetup{font=footnotesize}

\caption{F1 scores of CAC compared against baseline and $k$-means+X for different base classifiers on different datasets evaluated on a separate held out testing dataset.
Highlighted values signify the best result in that category.
}
\resizebox{\columnwidth}{!}{%
  \begin{tabular}{lccc}
    \toprule
    \textbf{Dataset/Classifier} & Diabetes & CIC & WID-M \\
      \midrule
        LR & {$0.111$} & {$0.443$} & {$0.398$} \\
        $KM$ + LR & {$0.503$} & {$0.492$} & {$0.41$} \\
        CAC + LR & {$\mathbf{0.505}$} & {$\mathbf{0.492}$} & {$\mathbf{0.41}$} \\
        \midrule
        
        SVM & {$\mathbf{0.523}$} & {$0.255$} & {$0.152$} \\
        $KM$ + SVM & {$0.116$} & {$0.426$} & {$\mathbf{0.354}$} \\
        CAC + SVM & {$0.506$} & {$\mathbf{0.452}$} & {$0.332$} \\
        \midrule
        
        XGB (n\_est=10) & {$0.236$} & {$0.454$} & {$0.41$} \\
        $KM$ + XGB & {$0.358$} & {$0.486$} & {$0.453$} \\
        CAC + XGB & {$\mathbf{0.506}$} & {$\mathbf{0.49}$} & {$\mathbf{0.457}$} \\
        \midrule
        
        LDA & {$0.112$} & {$0.479$} & {$\mathbf{0.46}$} \\
        $KM$ + LDA & {$0.24$} & {$\mathbf{0.497}$} & {$0.427$} \\
        CAC + LDA & {$\mathbf{0.506}$} & {$0.492$} & {$0.419$} \\
        \midrule
        
        Perceptron& {$0.349$} & {$0.421$} & {$\mathbf{0.346}$} \\
        $KM$ + Perceptron & {$0.333$} & {$0.375$} & {$0.342$} \\
        CAC + Perceptron & {$\mathbf{0.506}$} & {$\mathbf{0.45}$} & {$0.331$} \\
        \midrule
        
        RF (n\_tree=10) & {$0.406$} & {$0.308$} & {$0.345$} \\
        $KM$ + RF & {$0.407$} & {$0.313$} & {$\mathbf{0.347}$} \\
        CAC + RF & {$\mathbf{0.505}$} & {$\mathbf{0.449}$} & {$0.332$} \\
        \midrule
        
        KNN & {$0.406$} & {$0.218$} & {$0.219$} \\
        $KM$ + KNN & {$0.407$} & {$0.235$} & {$0.237$} \\
        CAC + KNN & {$\mathbf{0.506}$} & {$\mathbf{0.45}$} & {$\mathbf{0.331}$} \\
        \midrule
        
        SGD & {$0.204$} & {$0.433$} & {$0.387$} \\
        $KM$ + SGD& {$0.198$} & {$\mathbf{0.444}$} & {$0.383$} \\
        CAC + SGD & {$\mathbf{0.506}$} & {$0.432$} & {$\mathbf{0.393}$} \\
        \midrule
        
        Ridge & {$0.104$} & {$0.347$} & {$0.285$} \\
        $KM$ + Ridge & {$0.116$} & {$0.372$} & {$0.294$} \\
        CAC + Ridge & {$\mathbf{0.506}$} & {$\mathbf{0.45}$} & {$\mathbf{0.331}$} \\
        \bottomrule
  \end{tabular}
  }
  \label{tab:cac_f1}
\end{table}

\subsection{Results on Synthetic Data}
\label{app:synth_results_cac}

We analyse how the performance of CAC depends on the following data/algorithm parameters:
\begin{enumerate}[noitemsep,topsep=0pt,labelindent=0em,leftmargin=*]
    \item Inner Class Separation ($ICS$): Distance between class centroids in an individual cluster.
    \item Outer Cluster Separation ($OCS$): Distance between cluster centroids.
    \item Number of natural clusters in the dataset ($K$). 
    \item Number of classifiers used in CAC ($k$), i.e., the number of clusters used as input.
\end{enumerate}

To study the performance of CAC as a function of these $4$ parameters, we modify the \texttt{make classification} function of \textit{sklearn} library \cite{scikit-learn}, to generate custom datasets with varying ICS, OCS, and $K$. 
We choose Logistic Regression (LR) as the base classifier for CAC while testing its performance on the synthetic dataset.
The range of these parameters in our simulations is as follows: $ICS \in \{0, 0.2, 0.5, 1, 1.5, 2\}$, $OCS \in \{1, 1.5, 2\}$, $K \in \{2, 3, 5, 10, 15, 20, 30\}$ and $k \in \{2, 3, 4, 5\}$.
We run CAC+LR for all combinations of parameters.

Figure \ref{fig:cac_simulations} presents the results of the experiments. In each subfigure, the results are averaged over all the parameters not shown in the axes. In Figure \ref{fig:cac_simulations}a, note that the performance of CAC improves as Inner Class Separation within the clusters increases. This is in line with our hypothesis that if the classes are well separated within the clusters, then it will improve the performance of the classifiers. The performance of CAC does not depend much on the Outer Cluster Separation (Figure \ref{fig:cac_simulations}b). This is also expected as a classifier will be relatively unaffected by the presence of another cluster, irrespective of its distance as the data in that cluster is not a part of its training set.

As the number of natural clusters in the dataset increases, the performance of CAC decreases (Figure \ref{fig:cac_simulations}c).
This phenomenon is expected from linear classifiers as they cannot handle nonlinearities in data. Hence, dividing the data into multiple groups will decrease the piece-wise nonlinearity and the overall algorithm will work better \cite{gu2013clustered}. In Figure \ref{fig:cac_simulations}d, we observe that the results are better for the cases when $k \ge K$ (i.e., closer to the points near the x-axis). This shows that the performance is better when the number of clusters given as input is not less than the natural clusters in the data, compared to the case when the input number of clusters is less than the natural number of clusters.
In the latter case, each (linear) classifier has to learn a non-linear boundary due to the intrinsic cluster structure that leads to a decrease in the performance.

\begin{table*}[htb!]
\centering
\footnotesize
\captionsetup{font=footnotesize}

\caption{AUC of \deepmethod\ and other baselines. Best values in \textbf{bold}. KM, DCN, and IDEC use the same local classifier as used by \deepmethod.}
\label{tab:results_auc}
\resizebox{2\columnwidth}{!}{%
\begin{tabular}{ccccccccc}
\toprule
{\textbf{Dataset}} & \textbf{k} & \textbf{KM-Z} & \textbf{DCN-Z} & \textbf{IDEC-Z} & \textbf{DMNN} & \textbf{AC-TPC} & \textbf{GRASP} & \textbf{\deepmethod} \\
 & & $-$ & (2017) & (2017) & (2020) & (2020) & (2021) & (Ours) \\

\midrule

Diabetes & 2 & {$0.534 \pm 0.012$} & {$0.571 \pm 0.004$} & {$0.565 \pm 0.007$} & {$0.536 \pm 0.005$} & {$0.0 \pm 0.0$} & {$0.569 \pm 0.006$} & {$0.57 \pm 0.006$} \\ 
& 3 & {$0.517 \pm 0.02$} & {$0.569 \pm 0.003$} & {$0.558 \pm 0.012$} & {$0.546 \pm 0.007$} & {$0.540 \pm 0.0$} & {$0.565 \pm 0.005$} & {$\mathbf{0.573 \pm 0.004}$} \\
& 4 & {$0.508 \pm 0.027$} & {$0.569 \pm 0.003$} & {$0.564 \pm 0.007$} & {$0.545 \pm 0.008$} & {$0.5 \pm 0.0$} & {$0.561 \pm 0.004$} & {$0.568 \pm 0.005$} \\ 
\midrule 

CIC & 2 & {$0.606 \pm 0.036$} & {$0.7 \pm 0.054$} & {$0.766 \pm 0.053$} & {$0.737 \pm 0.039$} & {$0.0 \pm 0.0$} & {$\mathbf{0.827 \pm 0.015}$} & {$0.82 \pm 0.021$} \\ 
& 3 & {$0.535 \pm 0.059$} & {$0.705 \pm 0.033$} & {$0.657 \pm 0.039$} & {$0.762 \pm 0.026$} & {$0.757 \pm 0.0$} & {$0.81 \pm 0.022$} & {$0.815 \pm 0.027$} \\
& 4 & {$0.549 \pm 0.052$} & {$0.681 \pm 0.043$} & {$0.687 \pm 0.029$} & {$0.731 \pm 0.039$} & {$0.782 \pm 0.0$} & {$0.799 \pm 0.023$} & {$0.811 \pm 0.013$} \\ 
\midrule 

CIC-LoS & 2 & {$0.576 \pm 0.007$} & {$0.625 \pm 0.017$} & {$0.638 \pm 0.012$} & {$0.61 \pm 0.007$} & {$0.0 \pm 0.0$} & {$0.654 \pm 0.02$} & {$\mathbf{0.664 \pm 0.005}$} \\ 
& 3 & {$0.524 \pm 0.026$} & {$0.611 \pm 0.022$} & {$0.596 \pm 0.028$} & {$0.617 \pm 0.012$} & {$0.630 \pm 0.0$} & {$0.65 \pm 0.026$} & {$0.663 \pm 0.01$} \\
& 4 & {$0.512 \pm 0.02$} & {$0.625 \pm 0.025$} & {$0.623 \pm 0.012$} & {$0.615 \pm 0.009$} & {$0.601 \pm 0.0$} & {$0.639 \pm 0.025$} & {$0.663 \pm 0.005$} \\ 
\midrule 

ARDS & 2 & {$0.605 \pm 0.044$} & {$0.658 \pm 0.052$} & {$0.705 \pm 0.01$} & {$0.584 \pm 0.041$} & $-$ & {$0.457 \pm 0.03$} & {$\mathbf{0.716 \pm 0.022}$}\\ 
& 3 & {$0.546 \pm 0.044$} & {$0.704 \pm 0.022$} & {$0.648 \pm 0.077$} & {$0.571 \pm 0.036$} & {$0.655 \pm 0.019$} & {$0.443 \pm 0.029$} & {$0.705 \pm 0.019$}\\ 
& 4 & {$0.483 \pm 0.055$} & {$0.593 \pm 0.095$} & {$0.697 \pm 0.019$} & {$0.629 \pm 0.024$} & {$0.647 \pm 0.022$} & {$0.564 \pm 0.119$} & {$0.704 \pm 0.039$}\\ 
\midrule 

Sepsis & 2 & {$0.621 \pm 0.033$} & {$0.795 \pm 0.011$} & {$0.792 \pm 0.013$} & {$0.656 \pm 0.021$} & $-$ & {$0.495 \pm 0.029$} & {$\mathbf{0.803 \pm 0.014}$}\\ 
& 3 & {$0.55 \pm 0.09$} & {$0.782 \pm 0.014$} & {$0.781 \pm 0.007$} & {$0.603 \pm 0.055$} & {$0.714 \pm 0.034$} & {$0.474 \pm 0.008$} & {$0.793 \pm 0.034$}\\ 
& 4 & {$0.549 \pm 0.054$} & {$0.691 \pm 0.077$} & {$0.773 \pm 0.014$} & {$0.631 \pm 0.066$} & {$0.684 \pm 0.05$} & {$0.488 \pm 0.01$} & {$0.794 \pm 0.022$}\\ 
\midrule

WID-M & 2 & {$0.783 \pm 0.035$} & {$0.843 \pm 0.044$} & {$0.835 \pm 0.052$} & {$0.722 \pm 0.023$} & {$0.0 \pm 0.0$} & {$0.788 \pm 0.022$} & {$\mathbf{0.854 \pm 0.01}$} \\
& 3 & {$0.669 \pm 0.068$} & {$0.814 \pm 0.055$} & {$0.799 \pm 0.08$} & {$0.708 \pm 0.022$} & {$0.722 \pm 0.0$} & {$0.762 \pm 0.024$} & {$0.851 \pm 0.006$} \\
& 4 & {$0.584 \pm 0.105$} & {$0.795 \pm 0.059$} & {$0.811 \pm 0.053$} & {$0.716 \pm 0.019$} & {$0.734 \pm 0.0$} & {$0.79 \pm 0.015$} & {$0.847 \pm 0.01$} \\ 

\bottomrule
\end{tabular}
}
\end{table*}

\subsection{Results of CAC evaluated on Real Datasets}
\label{app:real_results_cac}

We use 9 different base classifiers, implemented by \textit{sklearn} library (BSD licenced): LR, Linear SVM, XGBoost (\#estimators = 10), Linear Discriminant Analysis (LDA), single Perceptron, Random Forest (\#trees = 10), $k$ Nearest Neighbors ($k=5$), SGD classifier and Ridge classifier (Ridge regressor predicting over the range $[-1,1]$) as a baseline directly and with $k$-means (KM+X) in a `cluster-then-predict' approach. We evaluate CAC on the Diabetes, WID-M and CIC datasets due to computational considerations related to large sizes of the ARDS and Sepsis datasets.

\section{\deepmethod\ Training Details}
\label{app:deepcac_training}

\textbf{Updating Clustering Parameters}
For optimizing the clustering loss, we follow the procedure defined in \cite{yang2017towards}. For fixed network parameters and cluster assignment matrix $M$, the assignment vectors of the current sample, i.e., $s_i$, are updated in an online fashion. Specifically, $s_i$ is updated as follows:

\[
    s_{i,j} = 
\begin{cases}
    1, & \text{if } j = \argmin_{l} \|f(x_i) - \mu_l\|^2 \\
    0, & \text{otherwise}
\end{cases}
\]

In \cite{yang2017towards}, the authors update the cluster centroids in an online manner instead of simply taking the average of every mini-batch as the current minibatch might not be representative of the global cluster structure. So the cluster centroids are updated as follows by a simple gradient step:
\[
\mu_j \xleftarrow{} \mu_j - \frac{1}{c_k^i} (\mu_j - f(x_i))s_{k,i}
\]

where $c_k^i$ is the count of the number of times algorithm assigned a sample to cluster $k$ before handling the incoming sample $x_i$. The class cluster centroids $\mu_k^b$ are also updated in a similar, online manner due to the above mentioned reasons.

\textbf{Updating Autoencoder's weights.}
For fixed $(M, {s_i})$, the subproblem of optimizing $(\Wcal, \Ucal)$ is similar to training an SAE – but with an additional loss term on clustering performance. Modern deep learning libraries like PyTorch allow us to easily backpropagate both the losses simultaneously. To implement SGD for updating the network parameters, we look at the problem w.r.t. the incoming data $x_i$:

\[
\min_{\Ucal, \Wcal} L^{i} = \ell(g(f(x_i)), x_i) +  \beta \|f(x_i) - Ms_i\| + \alpha \sum_{j=1}^{k} L_{AM}
\]

The gradient of the above function over the network parameters is easily computable, i.e., $\nabla X_{\Jcal} L_i = \frac{\partial \ell(g(f(x_i; \Wcal); \Ucal)}{\partial \Jcal} + \alpha \frac{\partial f(x_i)}{\partial \Jcal} (f(x_i) - Ms_i)$, where $\Jcal=(\Wcal, \Ucal, \Vcal)$ is a collection of network parameters and the gradients $\frac{\partial l}{\partial \Jcal}$ can be calculated by back-propagation. Then, the network parameters are updated by 

\[
\Jcal \xleftarrow{} \Jcal - \tau \nabla_{\Jcal} L^i
\]

where $\tau$ is the diminishing learning rate.

\begin{table}[]
\centering
\footnotesize
\captionsetup{font=footnotesize}
\caption{Most distinguishing features amongst clusters}
\begin{tabular}{llll}
\toprule
\textbf{Feature} & {\textbf{C1 Mean}} & {\textbf{C2 mean}} & {\textbf{C3 mean}} \\
\midrule
UrineOutputSum & 10.18 & 11.221 & 14.81 \\
MechVentLast8Hour & 0.424 & 0.524 & 0.903 \\
MechVentDuration & 1568.97 & 1762.353 & 2314.388 \\
WBC\_last & 10.355 & 12.284 & 13.734 \\
GCS\_median & 14.368 & 12.732 & 9.162 \\
Age & 62.202 & 67.46 & 62.134 \\
Creatinine\_last & 0.947 & 1.32 & 1.657 \\
BUN\_last & 17.159 & 25.647 & 30.013 \\
Albumin\_last & 3.091 & 2.984 & 2.756 \\
WBC\_first & 10.65 & 12.509 & 15.313 \\
\bottomrule
\label{tab:cluster_features}
\end{tabular}
\end{table}

\begin{table}[htb]
    \caption{Important Features for mortality prediction in respective clusters. Features common in all three clusters are highlighted in yellow while those common in two are highlighted in purple. The rest are unique to their respective clusters. Class ratios are depicted under cluster sizes.}
    \begin{tabular}{@{}lll@{}}
    \toprule
    $|C_1|=1856$ & $|C_2|=2845$ & $|C_3|=2049$ \\
    $6.25:2.07:1$ & $1.27:1.15:1$ & $1:1.14:3.07$ \\ \midrule
    WBC\_last & HCT\_last & SysABP\_first \\
    SOFA & BUN\_first & GCS\_last \\
    HCT\_first & \cellcolor{yellow!25}{PaO2\_last} & Age \\
    Platelets\_first & Creatinine\_first & \cellcolor{yellow!25}{PaO2\_last} \\
    Na\_last & NIMAP\_lowest & \cellcolor{blue!25}{Glucose\_last} \\
    HR\_highest & MechVentDuration & HR\_last \\
    \cellcolor{blue!25}{HR\_lowest} & Glucose\_highest & Na\_first \\
    NISysABP\_highest & GCS\_highest & PaCO2\_first \\
    \cellcolor{yellow!25}{PaO2\_last} & \cellcolor{blue!25}{HR\_lowest} & DiasABP\_lowest \\
    PaO2\_first & \cellcolor{blue!25}{Glucose\_last} & Weight \\
    \bottomrule
    \end{tabular}
    \label{tab:important_features}
\end{table}

\section{Extended Results}
\label{app:extended_results}

Table \ref{tab:results_auc}
present the AUC score (5-fold averaged on test dataset) of \deepmethod\ and baselines. \deepmethod\ beats GRASP in $5$ out of $6$ experiments.

\section{Case Study: ICU Length of Stay Prediction}
\label{app:case_study}
\looseness=-1

As a case study, we illustrate the use of \deepmethod\ in the CIC-LOS ICU Length of Stay (LoS) prediction data for $k=3$ clusters.
Since we are clustering with a view to induce class separability by making use of training labels, the clusters are influenced by the target label, i.e., LoS indicator, and thus by design \deepmethod\ is expected to find LoS subtypes.
In other words, we expect the inferred clusters to have different risk factors tailored to each underlying subpopulation.
To evaluate this, we examine feature importances for each cluster's local risk model.
We distill the knowledge of local networks 
into a simpler student model \cite{gou2021knowledge},
a Gradient Boosting Classifier in our case,
which provides importance values of characteristics. Note that the number of clusters need not be the same as number of classes in the target variable as is imposed by the cluster assumption.
In Table \ref{tab:important_features}, we list the cluster sizes $C_i$, the proportion of patients in the three quartiles of Length of Stay ($7, 13, 295$ hours) and the top 10 most important features in each cluster.
We see that, as expected, the risk models created by \deepmethod\ deem different sets of features as important 
for predicting mortality. Out of the $10$ most important features, there is 1 feature (PaO2\_last) that is common across all 3 clusters and 2 features are common across the two models. Otherwise, each model has $8$, $7$, and $8$ features unique to the respective clusters.

For the above case study, \deepmethod\ attains an AUPRC score of $0.495$ while the student model trained on the prediction scores of \deepmethod\ attains an AUPRC of $0.358$ (both evaluated on the same training and test datasets). $C_3$ has a majority of patients who spend a large time in the ICU while $C_1$'s majority population spent a relatively small time in the ICU.
In Table \ref{tab:cluster_features} we report the features that are most significant across the 3 clusters (determined using p-values). This indicates that the three clusters are most directly distinguished by the Mechanical Ventilation parameters and the UrineOutput. But within those clusters, the patient LoS depends on variables presented in Table \ref{tab:important_features}. We can cluster the dataset with any other variable also (say Age) and thus discover how LoS depends on different variables for various age groups.


\begin{table}[htb!]
\caption{AUC Scores for Time series datasets. Row-wise best result in bold. $^{\ast}p < 0.05$, $^{\ast\ast} p < 0.01$ indicates significantly better performance of \deepmethod-TS compared to the baseline.}
\begin{tabular}{ccccc}
\toprule
AUC & & ACTPC-TS & GRASP-TS & \deepmethod-TS \\
\midrule
ARDS-TS & 2 & - & $0.493$ & $\mathbf{0.615}^{\ast\ast}$ \\
ARDS-TS & 3 & - & $0.519$ & $\mathbf{0.605}^{\ast}$ \\
ARDS-TS & 4 & - & $0.505$ & $\mathbf{0.565}^{\ast}$ \\
\midrule
Sepsis-TS & 2 & - & $0.491$ & $\mathbf{0.76}^{\ast\ast}$ \\
Sepsis-TS & 3 & - & $0.528$ & $\mathbf{0.769}^{\ast}$ \\
Sepsis-TS & 4 & - & $0.465$ & $\mathbf{0.752}^{\ast\ast}$ \\
\bottomrule
\end{tabular}
\label{tab:auc_ts}
\end{table}



\section{\deepmethod-TS}
\label{app:deepcac_ts}

We extend \deepmethod, ACTPC and GRASP to handle time series data by replacing the Stacked Autoencoder with LSTM based autoencoder. $\alpha$ and $\beta$ for \deepmethod-TS are set to 20 and 5 respectively. For GRASP and ACTPC, we use the hyperparameters recommended in their original papers.

We evaluate all algorithms on the ARDS and Sepsis datasets as mentioned in table \ref{tab:datasets}. We use the first 24 hours of data to predict risk (of each condition, separately) in the remaining ICU stay, which is aligned with a hospital-centric schedule for prediction.
Risk prediction is formulated as a binary classification problem. The prediction task is to predict whether a patient will develop the condition at point in time after the 24 hours of stay in the ICU.
All patients whose length of ICU stay is $<24$ hours and those who develop the condition within 24 hours of their ICU stay are excluded. 

The AUC and AUPRC scores are mentioned in tables \ref{tab:auc_ts} and \ref{tab:auprc_ts}. We note that \deepmethod-TS performs better than GRASP for both the datasets and for all $k=2,3,4$. ACTPC-TS did not finish in time for both the datasets. All results are 3-fold averaged over the test dataset. The AUC and AUPRC scores for \deepmethod-TS are better than those obtained by \deepmethod\ on the static ARDS and Sepsis datasets as expected.

\end{document}

%% file: theory.tex
\section{Can Clustering Aid Classification?}
\label{sec:clustering_rademacher}

\looseness=-1
We provide a principled motivation behind our clustering approach by analyzing when clustering can potentially help obtain accurate classifiers. Our analysis yields crucial insights into developing the CAC framework. Fix $j\in \{1,\dots,k\}$ and let $S_j=\{(x_i, y_i)\}_{i=1}^{m_{j}}$ be the training dataset from the $j$-th cluster (having $m_j$ points) with 
$x \in\Xcal_{j}$ and $y \in \Ycal_{j}$ s.t. $\Xcal_{j} = \{x: x \in C_j\}$ and $\Ycal_{j} = \{y: y \in C_j\}$.
Then, the following theorem --- a direct application of a previous result \cite{bartlett2002rademacher,mohri2012foundations} to our problem --- suggests a potential benefit of clustering in terms of the expected error $\EE_{x,y}[\ell(f(x),y)]$ for unseen data: 

\begin{theorem}[Proof in Appendix \ref{app:thm_radamacher}]
\label{thm:clustering_rademacher}
Let $\Fcal_j$ be a set of maps $x \in C_j \mapsto f_{j}(x)$.  Let $\Fcal= \{x\mapsto f(x):f(x) = \sum_{j=1}^K \one\{x \in C_j\} f_j(x), f_j \in \Fcal_j\}$. Suppose that $0 \le \ell\left(q, y\right) \le \lambda_{j}$ for any $q \in\{f(x):f\in \Fcal_{j}, x\in C_j\}$ and  $y \in \Ycal_{j}$.  Then, for any $\delta>0$, with probability at least $1-\delta$, the following holds: for all maps $ f \in\Fcal$,
\begin{align}  \label{eq:new:1}
&\EE_{x,y}[\ell(f(x),y)]
\le \sum _{j=1}^K\Pr(x \in C_j)\cdot \biggl(\frac{1}{m_{j}}\sum_{i=1}^{m_{j}} \ell(f_{}(x_i^{j}),y_i^{j})+ \nonumber\\
&2\hat \Rcal_{m_{j}}(\ell \circ \Fcal_{j})+3\lambda_{j} \sqrt{\frac{\ln(K/\delta)}{2m_{j}}}\biggr),
\end{align}
where \\
$\hat \Rcal_{m_j}(\ell \circ \Fcal_{j}):=\EE_{\sigma}[\sup_{f_{j} \in \Fcal_{j}}\frac{1}{m_j} \sum_{i=1}^{m_j} \sigma_i \ell(f_{j}(x_{i}),y_{i})]$, $Pr(x \in C_j) = 1$ if $x\in C_j$ and $\sigma_1,\dots,\sigma_{m_j}$ are independent uniform random variables taking values in $\{-1,1\}$.
\end{theorem}

Theorem \ref{thm:clustering_rademacher} shows that the expected error $\EE_{x,y}[\ell(f(x),y)]$ for unseen data is bounded by three terms: the training error $\frac{1}{m_j}\sum_{i=1}^{m_j} \ell(f(x_i),y_i)$, the Rademacher complexity $\Rcal_{m_j}(\ell \circ \Fcal_{j})$ of the set of classifiers $\Fcal_{j}$ on the $j$-th cluster, and the last term $O({\ln(K/\delta)}/{2m_j})$ in equation \eqref{eq:new:1}. To apply Theorem \ref{thm:clustering_rademacher} to models without clustering, we can set the $j$-th cluster to contain all data points with $m_j=N$.

\looseness=-1
The upper bound can be used to analyze the benefit of using a simple model (e.g., a linear model) \textit{with} clustering compared to two cases of (1) a simple model and (2) a complex model (e.g., a deep neural network), both \textit{without} clustering. 
For a simple model \textit{without} clustering, the training error term $\frac{1}{m_j}\sum_{i=1}^{m_j} \ell(f(x_i^j),y_i^j)$ in equation \eqref{eq:new:1} can be large, because a simple model 
may not be able to sufficiently separate training data points (underfitting) to minimize the training error. 
In the case of a complex model \textit{without} clustering, the training error can be small but the  Rademacher complexity $\Rcal_{m_j}(\ell \circ \Fcal_{j})$ tends to be large. 
A simple model \textit{with} clustering can potentially trade-off these by minimizing Rademacher complexity $\Rcal_{m_j}(\ell \circ \Fcal_{j})$ while making the training data in the cluster linearly separable to minimize training error $\frac{1}{m_j}\sum_{i=1}^{m_j} \ell(f(x_i^j),y_i^j)$. This observation leads us to explore ways to induce class separability in data clusters.

\section{CAC}
\label{sec:cac}
\looseness=-1
The previous section discusses factors that give us hints of possible scenarios when clustering can aid classification. Using a simple model with clustering can be potentially advantageous when the first and second terms in the RHS of Equation \eqref{eq:new:1} are dominant. It thus motivates the formation of clusters, each of which can be linearly separable internally with respect to the classes for further improved performance of downstream simple classifiers. We test our idea by designing a simple algorithm, Clustering Aware Classification (CAC), that finds clusters to be used as training datasets by a chosen classifier. 
The key idea is to induce class separability within clusters to aid downstream cluster-specific classifiers.
During clustering, we measure the class separability by developing a metric based on class-specific centroids within each cluster, and theoretically analyze the effect of its use on the classification error.


\paragraph{Class Separability:}
We use a simple heuristic to measure class separability within each cluster.
For each cluster $C_j$ ($j\in [k]$), we define:
{\it cluster centroid}: $\mu(C_j) := |C_j|^{-1} \sum_{x\in C_j} x_i$,
{\it positive centroid}:  $\mu^{+}(C_j) := (\sum_{x_i \in C_j} y_i)^{-1} \sum_{x_i \in C_j} y_i x_i$, and
{\it negative centroid}: $\mu^{-}(C_j) := ({\sum_{x_i \in C_j} (1 - y_i)})^{-1} \sum_{x_i \in C_j} (1-y_i) x_i$.
Class separability is defined as the distance between positive and negative centroids within each cluster i.e. $\|\mu^{+} - \mu^{-}\|$.

\paragraph{Bounding Log-Loss of Classifiers in terms of Class Separation:}
\looseness=-1
We now show how class separation within clusters has a direct effect on the training error of classifiers employing log-loss. The following theorem also validates our choice of using $\|\mu^{+} - \mu^{-}\|$ as a heuristic to measure class separability.

\begin{theorem}[Proof in Appendix \ref{app:bounding_log_loss}]
\label{thm:bounding_log_loss}
Let $X = (x_1, x_2, \cdots, x_N)$ be a training dataset with $N$ records $x_i \in \mathbb{R}^{d}$ together with binary labels $y_i\in \left\{0,1\right\}$. Define the log-loss over the entire dataset as : $\ell(X) = -\Sigma_{i} y_i \ln(p(x_i)) + (1-y_i) \ln(1-p(x_i))$, where $p(x_i) = [1 + \exp(-\beta x_i)]^{-1}$ and $\beta = \argmin_{\beta} \ell(X)$. Then,
\[
C_1 - (C_3 \beta\mu^{+} - C_4\beta\mu^{-}) \leq \ell(X) \leq C_2 - (C_3\beta\mu^{+} - C_4\beta\mu^{-})
\]
where $C_1 = N\ln(2),\ C_2 = N\ln(1+\exp(c)) - \frac{Nc}{2},\ C_3 = \frac{N^{+}}{2},\ C_4 = \frac{N^{-}}{2},\ N^{+} = \sum_{x_i \in X} y_i$, $N^{-} = \sum_{x_i \in X} (1-y_i)$ and $c = \argmax_{i} \|\beta x_i\|$.
\end{theorem}

\paragraph{Clusters with Class Separability:}

To induce our notion of class separability within clustering we consider the $k$-means clustering formulation that is widely used for its simplicity, effectiveness, and speed.
Clusters using $k$-means are obtained by minimizing the size-weighted sample variance of each cluster: $\mathcal{D}(X) = \sum_{p=1}^k \inf_{\mu(C_p) \in X} \sum_{i \in C_p} \| x_i - \mu(C_p)\|^2.$
To induce class separability, a natural formulation is to add a cost with respect to class separability for each cluster.
With our measure for class separability, and 
$\alpha>0$ as a class separation hyperparameter,
we get $\phi(C_j) = \sum_{i\in C_j} \left(\|x_i - \mu(C_j) \|^2 - \alpha\cdot \|\mu^+(C_j) - \mu^-(C_j)\|^2\right)$.

Defining the cost function in this manner encourages the decrease of cluster variance term and increase of the class separability term, weighted by $\alpha$. Consequently, the overall cost function defined as $\phi(\{C_j\}_{1}^{k}) := \sum_{j=1}^{k} \phi(C_j)$ and the CAC algorithm aims to minimize that. We optimize for $\phi(\{C_j\}_{1}^{k})$ by using the Hartigan's method \cite{telgarsky2010hartigan} instead of the more common Lloyd's Algorithm \cite{lloyd1982least}. Our approach is summarized in Algorithm \ref{algo:cac}. $\Phi(x_i; C_p, C_q)$ represents the overall decrease in the cost function for the entire dataset if the point $x_i$ is moved from cluster $C_p$ to $C_q$ (see Appendix \ref{app:cac_details} for more details).

In conclusion, Theorem \ref{thm:bounding_log_loss} points out that under certain conditions (See Lemma \ref{lemma:neg_derivative} in Appendix \ref{app:lemma_neg_derivative}), the upper and lower bounds of log-loss provably decrease. The gap between the upper and lower bounds is a direct function of the breadth of the cluster. Thus as clusters become tighter, the training log-loss of the classifiers becomes more tightly bound.
Classification results on real datasets (see Table \ref{tab:cac_f1} in Appendix) show that CAC outperforms the simple cluster-then-predict approach on multiple base classifiers.

We however note that CAC has several limitations. Firstly, CAC assumes linear class separability which is not a realistic assumption in real world data.
Moreover, the restriction of working in the original data space leads to poor quality of clusters as compared to those found by plain $k$-means (c.f. Fig. \ref{fig:cac_sil_score} in Appendix). CAC initialized with $k$-means has a good Silhouette Index which decreases with training iterations). We thus turn to representation learning to discover transformations that conform to requirements of class separability and clustering while not being brittle at the same time.
This motivates the design of the neural variant of CAC, \deepmethod, which addresses these limitations by clustering in a low-dimensional embedding space instead of the data space, where the embeddings are obtained via non-linear dimensional reduction algorithms like deep autoencoders.

\begin{algorithm}[!hbt]
    \caption{CAC\ \label{algo:cac}}
    \KwIn{Training Data: $X \in \mathbb{R}^{N\times d}$, binary labels $y^{n \times 1} \in \{0,1\}^{N}$, $\{C_j\}_1^{k}$, $\{f_j\}_1^{k}$ and $\alpha$}
    {\bf $\triangleright$ Initialization} \\
    Compute $\mu(C_j), \mu^+(C_j), \mu^-(C_j)\ \forall\ C_j\ s.t.\ j \in \{1, \cdots, k\}$.\\
    {\bf $\triangleright$ Algorithm} \\
    \While{not converged}{
    \For{$i \in \{1 \hdots N\}$}{
    \If{Removing $x_i \in C_p$ from its cluster does not lead to a $1$-class cluster}{
        {\bf $\triangleright$ Assign new cluster to $x_i$} \\

        Let $q := \argmin_{j} {\Phi(x_i; C_p, C_j)}$ (breaking ties arbitrarily).\\
        \If{$\Phi(x_i, C_p, C_q)< 0$}
        {
            Assign $x_i$ to $C_q$ from $C_p$ and update $\mu(C_p)$, $\mu(C_q)$, $\mu^{\pm}(C_p)$ and $\mu^{\pm}(C_q)$.
        }
        Otherwise, let $x_i$ remain in $C_p$.
        }}
    }
    {\bf $\triangleright$ Train Classifiers} \\
    For every cluster $C_j$, train a classifier $f_j$ on $(\Xcal_{j}, \Ycal_{j})$.

    {\bf Output} Classifiers $\Fcal = \{f_j\}_{j=1}^{j=k}$ and cluster centroids $\mu = \{\mu_j\}_{j=1}^{j=k}$
\end{algorithm}